\documentclass[preprint,12pt]{elsarticle}

\usepackage[utf8]{inputenc}
\usepackage{amsthm}
\usepackage{booktabs}
\usepackage{float}
\usepackage[utf8]{inputenc}
\usepackage{dirtytalk}
\usepackage[dvipsnames]{xcolor}
\usepackage{algorithm}
\usepackage{svg}
\usepackage[noend]{algpseudocode}
\usepackage{subfig}
\usepackage{enumitem}
\usepackage{tikz}
\usepackage{multirow}
\usepackage[export]{adjustbox}
\usepackage{amsmath}
\usepackage{amssymb}
\usepackage{float}
\usepackage{svg}
\usepackage{enumitem}
\usepackage{makecell}
\usepackage{graphicx}
\usepackage{calrsfs}
\usepackage{pdflscape}
\usepackage{hyperref}

\theoremstyle{definition}
\newtheorem{definition}{Definition}[section]

\newtheorem{proposition}{Proposition}[section]

\DeclareMathOperator*{\argmax}{argmax}

\newcommand{\mathbbm}[1]{\text{\usefont{U}{bbm}{m}{n}#1}}
\newcommand\boldblue[1]{\textcolor{blue}{\textbf{#1}}}
\newcommand\boldred[1]{\textcolor{red}{\textbf{#1}}}

\journal{arxiv.org}

\begin{document}

\begin{frontmatter}



\title{Fuzzy granular approximation classifier}


\author{ Marko Palangeti\' c$^{a}$, Chris Cornelis$^a$,  Salvatore Greco$^{b,c}$, Roman S\l{}owi\' nski$^{d,e}$\\}

\address{$^a$Department of Applied Mathematics, Computer Science and Statistics, \\ Ghent University, Ghent, Belgium, \{marko.palangetic, chris.cornelis\}@ugent.be  \\
$^b$Department of Economics and Business, University of Catania, Catania, Italy,\\  salgreco@unict.it\\
$^c$Portsmouth Business School, Centre of Operations Research and Logistics (CORL), \\
University of Portsmouth, Portsmouth, United Kingdom \\
$^d$Institute of Computing Science,  Pozna\'n University of Technology, Pozna\'n, Poland,\\ roman.slowinski@cs.put.poznan.pl \\
$^e$Systems Research Institute, Polish Academy of Sciences, Warsaw, Poland \\}

\address{}

\begin{abstract}
In this article, a new Fuzzy Granular Approximation Classifier (FGAC) is introduced. The classifier is based on the previously introduced concept of the granular approximation and its multi-class classification case. The classifier is instance-based and its biggest advantage is its local transparency i.e., the ability to explain every individual prediction it makes. We first develop the FGAC for the binary classification case and the multi-class classification case and we discuss its variation that includes the Ordered Weighted Average (OWA) operators. Those variations of the FGAC are then empirically compared with other locally transparent ML methods. At the end, we discuss the transparency of the FGAC and its advantage over other locally transparent methods. We conclude that while the FGAC has similar predictive performance to other locally transparent ML models, its transparency can be superior in certain cases. 
\end{abstract}

\end{frontmatter}

\section{Introduction}
The concept of granular approximation is introduced in \cite{palangetic2021novel} and it is extended for the multi-class classification case in \cite{palangetic2022multi}. Granular approximation relies on the assumption that there is an underlying fuzzy membership degree of a particular data instance in the given decision class that is not observed and that should be estimated. The estimation is done using statistical optimization methods and using the assumption that if instance $u$ belongs to a certain class and $v$ is similar to $u$ then also $v$ should belong to the same class. The last assumption is known as the consistency property. The observed data are usually not consistent and the granular approximation also represents a novel relabeling of original data that satisfy the consistency property.

The motivation for introducing a membership degree in classification problems can be found in practice. For example, when we use a movie streaming service, we are often asked to rate a movie with like or dislike. We have only two options. But in reality, preference of movies is gradual; we like some movies more than others, but that graduality cannot be expressed with only two options: like and dislike. Hence, there exist a hidden preference degree that is not present in the observed data and that we try to estimate using granular approximations.

The goal of this article is to extend the granular approximation to new, unseen data. We design a classifier that estimates the membership degree of a new instance in a given decision class based on the consistency property. The name of the new classifier is \textit{Fuzzy Granular Approximation Classifier - FGAC}. The classifier is able to perform binary classification as well as multi-class classification natively. It belongs to the family of instance-based classifiers since the prediction is made based on the comparison of a new instance with those from the training set.

The main advantage of the classifier is its transparency which resides in the following two properties:
\begin{itemize}
    \item The explanation of the classifier can be derived from the ability to translate fuzzy logic into linguistic expressions.
    \item It is possible to identify the exact arguments that are in favour and against the prediction, as well as the strength of those arguments.
\end{itemize}

The remainder of this article is structured as follows.
In Section \ref{sec:preliminaries} we discuss the preliminaries of this article. Section \ref{sec:granular_approximations_recall} recalls the granular approximations as well as its version for the multi-class classification. In Section \ref{sec:making_predictions}, the novel Fuzzy Granular Approximation Classifier (FGAC) is introduced together with a version with OWA operators, while in Section \ref{sec:approx_calculation}, we discuss a way to speed up the training process, i.e., the calculation of the granular approximation. Section \ref{sec:experiments} contains empirical comparisons between different versions of FGAC and comparisons of FGAC with other ML models. In Section \ref{sec:transparency}, we explain why we consider FGAC as a transparent model as well as why we think that its transparency is superior compared with other ML models. Section \ref{sec:conclusion} concludes the article and outlines the future work.

\section{Preliminaries}
\label{sec:preliminaries}
\subsection{Fuzzy logic connectives}
In this subsection, the definitions and terminology are based on \cite{klement2013triangular}. Recall that {\em $t$-norm} $T:[0,1]^2 \rightarrow [0,1]$ is a binary operator which is commutative, associative, non-decreasing in both arguments, and $ \forall x \in [0,1] ,\, T(x,1) = x$. Since a $t$-norm is associative, we may extend it unambiguously to a $[0,1]^n \rightarrow [0,1]$ mapping for any $n > 2$. Some commonly used $t$-norms are listed in the left-hand side of Table \ref{table:tnorms}.

\begin{table}[H]
\begin{adjustbox}{width=\columnwidth,center}
    \begin{tabular}{c|rcl|rcl}
    Name & Definition & & & R-implicator\\
    \hline
Minimum &         $T_M(x,y)$ &=& $\min(x,y)$  & $I_{T_M}(x,y)$ &=& $\left\{\begin{array}{cc}
       1 & \mbox{if $x \le y$}  \\
            y & \mbox{otherwise} 
        \end{array}
        \right.$    \\
Product        & $T_P(x,y)$ &=& $xy$ & $I_{T_P}(x,y)$ &=& $\left\{\begin{array}{cc}
       1 & \mbox{if $x \le y$}  \\
            \frac{y}{x} & \mbox{otherwise} 
        \end{array}
        \right.$ \\
{\L}ukasiewicz    &    $T_L(x,y)$ &=& $\max(0,x+y-1)$ &    $I_{T_L}(x,y)$ &=& $\min(1,1-x+y)$ \\
Drastic     &   $T_D(x,y)$ &=& $\left\{\begin{array}{cc}
               \min(x,y) & \mbox{if $\max(x,y) = 1$}  \\
            0 & \mbox{otherwise} 
        \end{array}
        \right.$  &   $I_{T_D}(x,y)$ &=& $\left\{\begin{array}{cc}
           y & \mbox{if $x = 1$}  \\
            1 & \mbox{otherwise} 
        \end{array}
        \right.$\\
\makecell{Nilpotent \\ minimum} & $T_{nM}(x,y)$ &=& $\left\{\begin{array}{cc}
            \min(x,y) & \mbox{if $x + y > 1$}  \\
            0 & \mbox{otherwise} 
        \end{array}
        \right.$ & $I_{T_{nM}}(x,y)$ &=& $\left\{\begin{array}{cc}
            1 & \mbox{if $x \le y$}  \\
            \max(1-x,y) & \mbox{otherwise} 
        \end{array}
        \right.$\\
        \hline
    \end{tabular}
    \end{adjustbox}
    \vspace{8pt}
    \caption{Some common $t$-norms and their R-implicators}
    \label{table:tnorms}
\end{table}

A $t$-norm is an example of an aggregation operator. 
A binary aggregation operator $\mathbbm{A}: [0,1]^2 \rightarrow [0,1]$ (or just aggregation operator) is an operator which is non-decreasing in both arguments, and for which $\mathbbm{A}(0, 0) = 0$ and $\mathbbm{A}(1, 1) = 1$.
For $x, y \in [0,1]$, an aggregation operator is
\begin{itemize}
\item conjunctive if $\mathbbm{A}(x, y) \leq \min(x, y)$,
\item disjunctive if $\mathbbm{A}(x, y) \geq \max(x, y)$,
\item averaging if $\min(x,y) \leq \mathbbm{A}(x, y) \leq \max(x, y)$.
\end{itemize}
A $t$-norm is a conjunctive aggregation operator. 

An {\em implicator}  (or {\em fuzzy implication}) $I : [0,1]^2 \rightarrow [0,1]$ is a binary operator which is non-increasing in the first component, non-decreasing in the second one and such that $I(1,0) = 0$ and $I(0,0) = I(0,1) = I(1,1) = 1$. 
The residuation property holds for a $t$-norm $T$ and implicator $I$ if for all $x, y, z \in [0,1]$, it holds that
$$
T(x,y) \leq z \Leftrightarrow x \leq I(y,z).
$$ 
It is well-known that the residuation property holds if and only if $T$ is left-continuous and $I$ is defined as the residual implicator (R-implicator) of $T$, that is
$$I_T(x,y) = \sup \{\beta  \in [0,1] : T(x,\beta) \leq y\}.$$

The right-hand side of Table \ref{table:tnorms} shows the residual implicators of the corresponding $t$-norms. Note that all of them, except $I_{T_D}$, satisfy the residuation property.

If $T$ is a continuous $t$-norm, it is \textit{divisible}, i.e., for all $x, y \in [0,1]$, it holds that
$$
\min(x, y) = T(x, I(x, y)) = T(y, I(y, x)).
$$

A {\em negator}  (or { \em fuzzy negation}) $N : [0,1] \rightarrow [0,1]$ is a unary and non-increasing operator for which it holds that $N(0) = 1$ and $N(1) = 0$. A negator is involutive if $N(N(x)) = x$ for all $x \in [0,1]$. The standard negator, defined as 
$$
N_s(x) = 1-x,
$$ 
is involutive. 

For 
 implicator $I$, we define the negator induced by $I$ as $N(x) = I(x, 0)$. We will call triplet $(T, I, N)$, obtained as previously explained, a residual triplet. For a residual triplet we have that the following properties hold for all $x, y, z \in [0,1]$:
\begin{subequations}
\begin{align}
&\bullet    &&T(x, y) \leq x \quad \text{and} \quad T(x,y) \leq y, &&\label{eq:t-norm_smaller_parameters} \\
&\bullet    &&I(x, y) \geq y, 
&&\label{eq:implicator_greater_second_parameter} \\
&\bullet    &&T(x, I(x, y)) \leq y, 
&&\label{eq:modus_ponens} \\
&\bullet    &&x \leq y \Leftrightarrow I(x,y) = 1, \,  \text{(ordering property)} &&\label{eq:ordering_property} \\
&\bullet    &&T(x, I(y, z)) \leq I(I(x,y), z), &&\label{eq:t_norm_implicator_property2}\\
&\bullet    &&I(T(x,y), z) = I(x, I(y, z)), &&\label{eq:t_norm_implicator_property} \\
&\bullet    &&T(x, N(y)) \leq N(I(x,y)) \, \text{(consequence of (\ref{eq:t_norm_implicator_property2}) when $z=0$),} &&\label{eq:implicator_negator_property}\\
&\bullet    &&N(T(x,y)) = I(x, N(y)) \, \text{(consequence of (\ref{eq:t_norm_implicator_property}) when $z=0$)}.
&&\label{eq:implicator_negator_property2}
\end{align}
\end{subequations}

For a given involutive negator $N$, we say that aggregation operator $\mathbbm{A}$ is $N$-invariant if
\begin{equation}
    \label{eq:n_invariance}
    \mathbbm{A}(x, y) = N(\mathbbm{A}(N(x), N(y)))
\end{equation}
It is easy to verify that conjunctive and disjunctive operators cannot be $N$-invariant. 

Two fuzzy binary operators $B^1$ and $B^2$ are isomorphic if there exists a bijection $\varphi:[0,1] \rightarrow [0,1]$ such that $B^1 = \varphi^{-1}(B^2(\varphi(x), \varphi(y)))$ while unary operators $V^1$ and $V^2$ are isomorphic if $V^1 = \varphi^{-1}(V^2(\varphi(x)))$. Moreover, we write $B^1 \equiv B^2_{\varphi}$ and $V^1 \equiv V^2_{\varphi}$.

Every involutive negator is isomporphic to the standard negator $N_s$.

If residual triplet $(T, I, N)$ is generated by $t$-norm $T$, then the residual triplet generated by $T_{\varphi}$ is $(T_{\varphi}, I_{\varphi}, N_{\varphi})$.

A $t$-norm for which the induced negator of its R-implicator is involutive is called an IMTL $t$-norm. In Table \ref{table:tnorms}, $T_L$ and $T_{nM}$ are IMTL $t$-norms where the corresponding induced negator is $N_s$. A residual triplet $(T, I, N)$ that is generated with an IMTL $t$-norm is called an IMTL triplet. If $(T, I, N)$ is an IMTL triplet, then $(T_{\varphi}, I_{\varphi}, N_{\varphi})$ is also an IMTL triplet. 

For an IMTL triplet, we have that the following properties hold for all $x, y, z \in [0,1]$:
\begin{subequations}
\begin{align}
&\bullet    &&I(N(x), N(y)) = I(y, x), &&\label{eq:implicator_negator_property3}\\
&\bullet    &&T(x, N(y)) = N(I(x,y)). &&\label{eq:implicator_negator_property4}
\end{align}
\end{subequations}

A continuous $t$-norm is IMTL if and only if it is isomorphic to the Łukasiewicz $t$-norm. Such $t$-norm is \textit{strongly max-definable}, i.e., for all $x,y \in [0,1]$, it holds that
\begin{equation}
\label{eq:strong_max_definability}
\max(x, y) = I(I(x, y), y) = I(I(y, x), x).
\end{equation}
A residual triplet generated by a $t$-norm isomorphic to $T_L$, $T_{L,\varphi}$ is denoted by $(T_{L,\varphi}, I_{L, \varphi}, N_{L, \varphi})$. Note that $N_L \equiv N_s$.

\subsection{Fuzzy sets and fuzzy relations}
\label{subsec:fuzzy_sets_relations}
Given a non-empty universe set $U$, a fuzzy set $A$ on $U$ is an ordered pair $(U, m_A)$, where $m_A:U \rightarrow [0,1]$ is a membership function that indicates how much an element from $U$ is contained in $A$. Instead of $m_A(u)$, the membership degree is often written as $A(u)$. If the image of $m_A$ is $\{0,1\}$ then $A$ is a crisp (ordinary) set. For negator $N$, the fuzzy complement $coA$ is defined as $coA(u) = N(A(u))$ for $u\in U$. If $A$ is crisp then $coA$ reduces to the standard complement. For $\alpha \in (0,1]$, the $\alpha$-level set of fuzzy set $A$ is a crisp set defined as $A_{\alpha} = \{u \in U; A(u) \geq \alpha \}$.

A fuzzy relation $\widetilde{R}$ on $U$ is a fuzzy set on $U \times U$, i.e., a mapping $\widetilde{R}: U \times U \rightarrow [0,1]$ which indicates how much two elements from $U$ are related. Some relevant properties of fuzzy relations include:
\begin{itemize}
    \item $\widetilde{R}$ is reflexive if $\forall u \in U, \widetilde{R}(u,u) = 1$.
    \item $\widetilde{R}$ is symmetric if $\forall u,v \in U,\ \widetilde{R}(u,v) = \widetilde{R}(v,u)$.
    \item $\widetilde{R}$ is $T$-transitive w.r.t. $t$-norm $T$ if $\forall u,v,w \in U$ it holds that \\
    $T(\widetilde{R}(u,v),\widetilde{R}(v,w)) \leq \widetilde{R}(u,w)$.
\end{itemize}{}
A reflexive and $T$-transitive fuzzy relation is called a $T$-preorder relation while a symmetric $T$-preorder is a $T$-equivalence relation.

To illustrate some of these fuzzy relations, we assume that instances from $U$ are described with a finite set of numerical attributes $Q$. For attribute $q \in Q$, Let $u^{(q)}$ and $v^{(q)}$ be the evaluations of instances $u$ and $v$ on attribute $q$.
An example of a $T_L$-preorder relation (expressing dominance) on attribute $q$, given in \cite{palangetic2021fuzzy}, is
\begin{equation}
    \label{eq:triangular_dominance}
    \widetilde{R}_q^\gamma(u,v) = \max \left (\min \left (1 -\gamma \frac{v^{(q)} - u^{(q)}}{range(q)},1 \right ), 0 \right),
\end{equation}{}
where $\gamma$ is a positive parameter and $range(q)$ is the difference between the maximal and minimal value on $q$. 
An example of a $T_L$-equivalence relation (expressing indiscernibility) on the same attribute is
\begin{align}
\label{eq: triangular similarity}
   \widetilde{R}_q^{\gamma} (u,v) = \max \left ( 1 - \gamma \frac{|u^{(q)} - v^{(q)}|}{range(q)},0\right).
\end{align}
In both cases, the relation over all attributes from $Q$ is defined as $\widetilde{R}(u,v) = \min_{q \in Q} \widetilde{R}_q(u,v)$.

$T_L$-equivalence (\ref{eq: triangular similarity}) can be generalized to the broader family of $T_L$-equivalences based on distance metrics. For a given distance metrics $d$ on $U$, a $T_L$ equivalence based on $d$ can be defined as:
\begin{equation}
\label{eq:general_equivalence}
 \widetilde{R}_d^{\gamma} (u,v) = \max \left ( 1 - \gamma \cdot d(u,v), 0 \right).
\end{equation}
Then, $T_L$-equivalence (\ref{eq: triangular similarity}) is a member of family (\ref{eq:general_equivalence}) for $d$ being distance based on the supremum norm (or $l_{\infty}$ norm) applied on scaled attributes (division by $range(q)$). Also, in our experiments we will use a relation from (\ref{eq:general_equivalence}) based on the Euclidean distance applied on the scaled attributes. We will call those two similarity measures supremum similarity and Euclidean similarity respectively. 

In the definition of the Euclidean similarity, we may observe that the distance can become large if the number of attributes is high which leads to the low values of the similarity relation. In order to avoid that, it is desirable to average the distances over individual attributes, i.e., instead of $d(u, v)$ we use $\frac{d(u,v)}{\sqrt{|Q|}}$ and that value will be multiplied with $\gamma$. We use this version of the definition throughout the paper. 

\subsection{Fuzzy rough and granularly representable sets}
This subsection is based on \cite{palangetic2021granular}.
Let $U$ be a finite set of instances, $A$ a fuzzy set on $U$ and $\widetilde{R}$ a $T$-preorder relation on $U$. The fuzzy PRSA lower and upper approximations of $A$ are fuzzy sets for which the membership function is defined as:
\begin{equation}
\label{eq:fuzzy rough approx}
\begin{aligned}
\underline{\text{apr}}_{R}^{\min, I}(A)(u) = \min \{I(\widetilde{R}(v,u), A(v)); v \in U\} \\
\overline{\text{apr}}_{R}^{\max, T}(A)(u) = \max \{T(\widetilde{R}(u,v), A(v)); v \in U\}.
\end{aligned}
\end{equation}
For an IMTL $t$-norm and the corresponding implicator $I$ and negator $N$, we have the well known duality property.
\begin{align}
\label{eq:duality}
    \begin{split}
    N (\underline{\text{apr}}_{R}^{\min, I}(A)(u) ) &=  \overline{\text{apr}}_{R}^{\max, T}(coA)(u) \\ 
    N (\overline{\text{apr}}_{R}^{\max, T}(A) (u)) &=  
    \underline{\text{apr}}_{R}^{\min, I}(coA) (u)
    \end{split}
\end{align}

A fuzzy granule with respect to fuzzy relation $\widetilde{R}$ and parameter $\lambda \in [0,1]$ is defined as a parametric fuzzy set
\begin{align}
\label{eq:granule}
    \widetilde{R}^+_\lambda (u) = \{(v, T(\widetilde{R}(v,u), \lambda ); v \in U\}.
\end{align}
while the granule with respect to the inverse fuzzy relation $\widetilde{R}^{-1}$ is 
\begin{align}
\label{eq:granule2}
    \widetilde{R}^-_\lambda (u) = \{(v, T(\widetilde{R}(u,v), \lambda ); v \in U\}.
\end{align}
Here, parameter $\lambda$ describes the association of instance $u$ to a particular decision. For example, in classification problems, it represents the membership degree of $u$ to  a particular decision class. 
A fuzzy set $A$ is granularly representable (GR) w.r.t. relation $\widetilde{R}$ if 
$$
A = \bigcup \{ \widetilde{R}^+_{A(u)}(u); u \in U\},
$$
where the union is defined using $\max$ operator.
Some equivalent forms to define granular representability are such that for all $u,v \in U$:
\begin{equation}
\label{eq:granular_rep}
T(\widetilde{R}(v,u), A(u)) \leq A(v) \Leftrightarrow \widetilde{R}(v,u) \leq I(A(u), A(v)) 
\end{equation}

\begin{proposition} \cite{palangetic2021novel}
\label{prop:complement_representable}
If fuzzy set $A$ is granularly representable w.r.t. $T$-preorder relation $\Tilde{R}$, then $coA$ is granularly representable w.r.t.  $\Tilde{R}^{-1}$.
\end{proposition}
\begin{proposition}
\label{prop:l_and_s_gp} \cite{palangetic2021granular}
It holds that $\underline{\text{apr}}_{\Tilde{R}}^{\min, I}(A)$ is the largest GR set contained in $A$, while $\overline{\text{apr}}_{\Tilde{R}}^{\max, T}(A)$ is the smallest GR set containing $A$.
\end{proposition}

\subsection{Ordered weighted average}
\label{subsec:owa}
In order to avoid exclusive influence of extreme values (minima and maxima) in decision making, an Ordered Weighted Average (OWA) operator is introduced. While keeping high influence of the extrema, OWA operator also utilizes the values that are non-extreme. It can be seen as a "softer" version of minima and maxima. The OWA aggregation of set $V$ of $n$ real numbers with weight vector $W=  (w_1,w_2, . . . ,w_{n})$, where $w_i \in [0, 1]$ and $\Sigma_{i = 1}^{n} w_i = 1$, is given by
$$
\text{OWA}_W(V) = \sum_{i = 1}^{n} w_i  v_{(i)},
$$
where $v_{(i)}$ is the $i$-th largest element in the set $V$.
Different weight vectors are used depending if they will be used to replace $\min$ or $\max$ operator. Those operators can be expressed trough OWA operators with the corresponding weights:
\begin{equation*}
W_{\min}  = (0, \dots, 0, 1), \quad W_{\max} = (1, 0, \dots, 0)
\end{equation*}
We say that these weights are complementary i.e., it holds that $(W_{\min})_i = (W_{\max})_{n - i + 1}$. We denote the complementarity with $W_{\min} = \overline{W_{\max}}$ and $W_{\max} = \overline{W_{\min}}$.
Denote with $W_{L}$ the weights used to replace $\min$ and with $W_U$ weights used to replace $\max$. The well-known weights used in practice are
\begin{itemize}
    \item additive: $W_{L}^{add} = (\frac{2}{n(n +1)}, \frac{4}{n   (n + 1 )  }, \dots, \frac{2(n-1)}{n(n + 1)} ,  \frac{2}{n + 1})$, $W^{add}_U = \overline{W^{add}_L}$,
    \item exponential: $W_{L}^{exp} = (\frac{1}{2^n - 1}, \frac{2}{2^n - 1}, \dots, \frac{2^{n-2}}{2^n - 1} ,  \frac{2^{n-1}}{2^n - 1})$, $W^{exp}_U = \overline{W^{exp}_L}$,
    \item inverse additive: $W_{L}^{invadd} = (\frac{1}{nD_n}, \frac{1}{(n-1)D_n}, \dots, \frac{1}{2D_n} ,  \frac{1}{D_n})$, $W^{invadd}_U = \overline{W^{invadd}_L}$ for $D_n = \sum_{i = 1}^{n} \frac{1}{i}$.
\end{itemize}

\subsection{Datasets}
\label{subsec:datasets}

We present the datasets that will be used in the experiments. We collected 18 classification datasets that are available in the UCI Machine Learning repository \cite{ucidatasets2019}. Their description is provided in Table \ref{tab:datasets_description}.

\begin{table}[H]
    \begin{adjustbox}{width=\columnwidth,center}
    \begin{tabular}{lrrrrl}
\toprule
        name &  \# of instances & \makecell{\# of numerical \\ attributes} & \makecell{ \# of nominal \\ attributes} &  \# of classes & \makecell{distribution of \\ instances among classes} \\
\midrule
  australian &             690 &                          8 &                        6 &             2 &                                   (383, 307) \\
     balance &             625 &                          4 &                        0 &             3 &                               (288, 288, 49) \\
      breast &             277 &                          0 &                        9 &             2 &                                    (196, 81) \\
        bupa &             345 &                          6 &                        0 &             2 &                                   (200, 145) \\
   cleveland &             297 &                         13 &                        0 &             5 &                        (160, 54, 35, 35, 13) \\
         crx &             653 &                          6 &                        9 &             2 &                                   (357, 296) \\
      german &            1000 &                          7 &                       13 &             2 &                                   (700, 300) \\
       glass &             214 &                          9 &                        0 &             6 &                      (76, 70, 29, 17, 13, 9) \\
    haberman &             306 &                          3 &                        0 &             2 &                                    (225, 81) \\
       heart &             270 &                         13 &                        0 &             2 &                                   (150, 120) \\
  ionosphere &             351 &                         33 &                        0 &             2 &                                   (225, 126) \\
mammographic &             830 &                          5 &                        0 &             2 &                                   (427, 403) \\
        pima &             768 &                          8 &                        0 &             2 &                                   (500, 268) \\
     saheart &             462 &                          8 &                        1 &             2 &                                   (302, 160) \\
 spectfheart &             267 &                         44 &                        0 &             2 &                                    (212, 55) \\
       vowel &             990 &                         13 &                        0 &            11 & \makecell{ (90, 90, 90, 90, 90, \\ 90, 90, 90, 90, 90, 90)} \\
        wdbc &             569 &                         30 &                        0 &             2 &                                   (357, 212) \\
   wisconsin &             683 &                          9 &                        0 &             2 &                                   (444, 239) \\
\bottomrule
\end{tabular}
\end{adjustbox}
    \caption{Description of datasets}
    \label{tab:datasets_description}
\end{table}
We may notice that some datasets have nominal features which have to be encoded into numerical ones. In general, we use One Hot Encoding \cite{garavaglia1998smart} for this purpose. If a ML method that uses distance metrics is applied, a nominal attribute contributes to the distance with value 1 if the category of two instances is different, and 0 if they are the same. Also, since we use the factor $\frac{1}{\sqrt{|Q|}}$ in the definition of Euclidean similarity, where value $|Q|$ is the one we obtain after One Hot Encoding is applied. 

\subsection{Methods to compare with}
\label{subsec:models_to_compare_with}
The performance of FGAC will be compared with other popular classification models. Due to the local transparency of FGAC, that will be discussed later, we will compare it with  well-known simple and (up to some degree) locally transparent ML algorithms that are widely in use. These methods are k-Nearest Neigbours \cite{fix1989discriminatory}, Classification and Regression Tree \cite{quinlan2014c4} and Learning Vector Quantization \cite{kohonen1995learning}.

k-Nearest Neigbours (kNN) is a non-parametric lazy approach where the decision for a new particular instance is obtained based on the majority decisions of the $k$ closest instances w.r.t. a given distance metrics. The transparency of this approach boils down to our ability to detect the instances based on which the decision was made. However, the transparency fades as $k$ increases because it becomes hard to understand how a prediction was made based on a high number of other instances.

Classification And Regression Tree (CART) can be seen as a hierarchical rule-based model. In every step of the training phase, a split of the training set of instances is performed based on a provided criterion. In the first step, the whole set of instances is split, while in every subsequent step, a subset of the previous split is chosen and split. This way of splitting creates a binary decision tree. Since every split is performed on a specific attribute, a hierarchical set of rules can be induced in order to explain a particular prediction made by CART. These rules enable the transparency of the model.

Learning Vector Quantization (LVQ) is a prototype based model, where for each decision class a few points from the attribute space called \textit{prototypes} are learned. These prototypes do not necessarily coincide with the training instances. After the prototypes are learned, a new instance is classified based on the decision of the nearest prototype. The transparency of  LVQ lies in the fact that one is able to identify the prototype responsible for the prediction. 

We will also compare our model with the family of the k-Fuzzy-Rough Nearest Neighbour (kFRNN) methods as another family of methods based on fuzzy logic and consistency in data \cite{jensen2011fuzzy, ramentol2014ifrowann}. It is a lazy approach were for every new instance and for every decision class, we calculate its fuzzy rough lower approximation degree, upper approximation degree and take the mean as the membership degree in that decision class. Then, the decision class is determined as the one for which the highest membership degree is achieved. kFRNN also invokes OWA operators as a replacement for min and max operators in the lower and upper approximations. 

\section{Training procedure and granular approximations}
\label{sec:granular_approximations_recall}
\subsection{General case}
This section also recalls previously known results, but because of its importance, we separate it from Preliminaries. The section is based on the results from  \cite{palangetic2021novel} and \cite{palangetic2022multi}.
A granular approximation is a granularly representable set that is as close as possible to the observed fuzzy set (set that is approximated) with respect to the given closeness criterion. The closeness is measured by a loss function $L:\mathbb{R} \times \mathbb{R} \rightarrow \mathbb{R}^+$. For a given loss function $L$, observed fuzzy set $\Bar{A}$, relation $\widetilde{R}$ and residual triplet $(T, I, N)$, the granular approximation $\hat{A}$ is obtained as a result of the following optimization problem:
\begin{equation}
\label{eq:general_optimization}
\begin{aligned}
&\text{minimize}  && \displaystyle\sum_{u \in U}  L(\Bar{A}(u), \hat{A}(u)) \\
&\text{subject to}    && T( \widetilde{R}(u,v), \hat{A}(v)) \leq \hat{A}(u), \quad   u, v \in U\\
  &              &&0 \leq \hat{A}(u) \leq 1, \quad u \in U.
\end{aligned}
\end{equation}
The objective function in (\ref{eq:general_optimization}) ensures that the resulting fuzzy set $\hat{A}$ is as close as possible to the given fuzzy set $A$ (w.r.t. loss function $L$) while the constraints of (\ref{eq:general_optimization}) guarantee that $\hat{A}$ is granularly representable. Value of $\hat{A}$ stands for an estimated membership degree of $u \in A$.

We recall two well-known loss functions; $p$-quantile loss:
\begin{equation}
    L_{p} = (y, \hat{y}) = (y - \hat{y}) (p - \mathbf{1}_{y-\hat{y} < 0}) = \begin{cases}
        p|y - \hat{y}| & \text{if   }  y-\hat{y} > 0, \\
        (1-p)|y - \hat{y}| & \text{otherwise},
        \end{cases} \label{eq:quantile_loss}
\end{equation}
and mean squared error:
\begin{equation}
L_{MSE}(y, \hat{y}) = (y - \hat{y})^2 \label{eq:mse}.
\end{equation}
The $p$-quantile loss for $p=\frac{1}{2}$ is called mean absolute error:
\begin{equation}
    L_{MAE}(y, \hat{y}) = |y - \hat{y}| \label{eq:mae}
\end{equation}

It was shown that optimization problem (\ref{eq:general_optimization}) can be efficiently solved if $T$ is isomorphic to $T_L$ ($T = T_{L,\varphi}$) and for $L$ being the scaled $p$-quantile loss: $L_{p,\varphi} = L_p(\varphi(y), \varphi(\hat{y}))$ or the scaled mean squared error: $L_{MSE,\varphi} = L_{MSE}(\varphi(y), \varphi(\hat{y}))$. In the case of $L_{p,\varphi}$, problem (\ref{eq:general_optimization}) becomes a linear program:
\begin{equation}
\label{eq:T_L_linear}
\begin{aligned}
&\text{minimize}   && p\displaystyle\sum_{u \in U}x_u + (1-p)\displaystyle\sum_{u \in U}y_u, &&\\
&\text{subject to}    &&\alpha_u - \alpha_v + 1 \geq \widetilde{R}_{\varphi}(u,v),  && u, v \in U\\
        &    &&x_u - y_u =\Bar{A}_{\varphi}(u) - \alpha_u,  && u \in U\\
  &              && x_u \geq 0, \, y_u\geq 0. &&u \in U
\end{aligned}
\end{equation}
where $x_u = \max(\varphi(\Bar{A}(u) - \alpha_u ), 0)$, $y_u =  \max( \alpha_u - \varphi(\Bar{A}(u)), 0)$, $\Bar{A}_{\varphi}(u) = \varphi(\Bar{A}(u))$ and $\widetilde{R}_{\varphi}(u,v) = \varphi(\widetilde{R}(u,v))$. In the case of $L_{MSE,\varphi}$, problem (\ref{eq:general_optimization}) becomes a quadratic program:
\begin{equation}
\label{eq:T_L_quadratic}
\begin{aligned}
&\text{minimize}  && \displaystyle\sum_{u \in U} (\alpha_u - \Bar{A}_{\varphi}(u))^2, &&\\
&\text{subject to}    &&\alpha_u - \alpha_v + 1 \geq \widetilde{R}_{\varphi}(u,v),  && u, v \in U.\\
\end{aligned}
\end{equation}

\begin{definition}
Loss function $L$ is \textit{symmetric} if $L(y, \hat{y}) = L(\hat{y}, y)$.
\end{definition}
It is easy to verify that $L_{MSE,\varphi}$ and $L_{MAE, \varphi}$ are symmetric loss functions, while $L_{p, \varphi}$ for $p \neq \frac{1}{2}$ is not. However, it can be observed that $L_{p, \varphi}(y, \hat{y}) = L_{1 - p, \varphi}(\hat{y}, y)$.

\begin{definition}
\label{def:v-type}
We say that loss function $L$ is \textit{of $\lor$-type} if for any real number $a$, it holds that
\begin{itemize}
    \item $L(a,a) = 0$,
    \item functions $L(x,a)$ and $L(a,x)$ are increasing for $x > a$ and
    \item functions $L(x,a)$ and $L(a,x)$ are decreasing for $x < a$.
\end{itemize}
\end{definition}
The previous definition says that the loss is greater if $x$ is more distant from $a$. It is easy to verify that the mean squared error and $p$-quantile loss for $p \in (0,1)$ are of $\lor$-type. The $p$-quantile loss for $p \in \{0, 1\}$ is not of $\lor$-type since $L_0(a,x) = 0$ for $x < a$ and $L_1(a, x) = 0$ for $x > a$. 

\begin{definition}
\label{def:perserving_duality}
A loss function $L:[0,1] \times [0,1] \rightarrow \mathbb{R}^+$ is \textit{$N$-duality preserving} if $L(y, \hat{y}) = L(N(\hat{y}), N(y))$ for $N$ from the residual triplet $(T, I, N)$.
\end{definition}
In \cite{palangetic2021novel}, it was shown that both $L_{p, \varphi}$ and $L_{MSE,\varphi}$ are $N$-duality preserving for IMTL triplet $(T_{L, \varphi}, I_{L, \varphi}, N_{L, \varphi})$.

\subsection{Classification case}

In \cite{palangetic2022multi}, binary and multi-class classification cases were considered. In the binary classification case, it is assumed that observed set $\Bar{A}$ from (\ref{eq:general_optimization}) is crisp. It is encoded with a fuzzy set with the same notation in a way that $\Bar{A}(u) = 1$ if $u \in \Bar{A}$ and $\Bar{A}(u) = 0$ if $u \in co\Bar{A}$. 

In such case, for $u \in \Bar{A}, v \in co\Bar{A}$ we have the following property:
\begin{equation}
\label{eq:adjacent_granules}
    \hat{A}(u) = \min_{w \in coA} I(\widetilde{R}(w, u), \hat{A}(w)), \quad \hat{A}(v) = \max_{w \in A} T(\widetilde{R}(v, w), \hat{A}(w)).
\end{equation}

Denote $\beta_u = \hat{A}(u)$ for $u \in A$ and $\beta_u = N(\hat{A}(u))$ for $u \in coA$. We refer to the new notation as an \textit{alternative notation}. While $\hat{A}(u)$ stands for an estimated membership degree of $u$ in $A$, $\beta_u$ is an estimated membership degree of $u$ in the observed class of $u$ (it can be either $A$ or $coA$).

The alternative notation is important to extend the optimization procedure (\ref{eq:general_optimization}) to the multi-class classification case. 

Let $L$ be of $\lor$-type and $N$-dual preserving and symmetric, let $\widetilde{R}(u,v)$ be a $T$-equivalence and let $S$ be crisp equivalence relation on $U$ defined as $S(u,v) = 1$ if $u$ and $v$ are from the same observed decision class, i.e., $\Bar{A}(u) = \Bar{A}(v)$, and $S(u,v)=0$ otherwise.
In this case, problem (\ref{eq:general_optimization}) can be reformulated as:

\begin{equation}
\label{eq:disjoint_optimization}
\begin{aligned}
&\text{minimize}  && \displaystyle\sum_{u \in U} L(1, \beta_u) \\
&\text{subject to} && T(\beta_u, \beta_v) \leq I(\widetilde{R}(u,v), S(u,v))),  \quad   u , v \in U\\
  &              &&0 \leq \beta_u \leq 1, \quad u \in U.
\end{aligned}
\end{equation}
Since relation $S$ can be extended to distinguish among more than only 2 classes, the previous form is also suitable for the multi-class classification problems. Also, the interpretation of the alternative notation is extended in the same way. Value $\beta_u$ stands for the membership degree of $u$ in the observed class of $u$.

As in the case of problem (\ref{eq:general_optimization}), problem (\ref{eq:disjoint_optimization}) can be efficiently solved if $T$ is isomorphic to $T_L$ and for $L$ being $L_{MSE, \varphi}$ or $L_{MAE, \varphi}$. In the case of $L_{MAE, \varphi}$, problem (\ref{eq:disjoint_optimization}) becomes a linear program:

\begin{equation}
\label{eq:T_L_linear_multiclass}
\begin{array}{ll@{}ll}
&\text{maximize}  && \displaystyle\sum_{u \in U}  \alpha_u \\
&\text{subject to}    &&\alpha_u + \alpha_v \leq 1 + M_{\varphi}(u,v), \quad  u, v \in U\\
  &              && 0 \leq \alpha_u \leq 1, \quad u \in U,
\end{array}
\end{equation}

while in the case of $L_{MSE, \varphi}$ we have a quadratic program:
\begin{equation}
\label{eq:T_L_quadratic_multiclass}
\begin{array}{ll@{}ll}
&\text{maximize}  && \displaystyle\sum_{u \in U} (1 - \alpha_u)^2\\
&\text{subject to}    &&\alpha_u + \alpha_v \leq 1 + M_{\varphi}(u,v), \quad  u, v \in U\\
  &              && 0 \leq \alpha_u \leq 1, \quad u \in U.
\end{array}
\end{equation}
In the following section, we assume that we deal with a classification problem. First, we will start with the binary case and then extend it to the multi-class case. 

\section{Prediction for unseen objects}
\label{sec:making_predictions}

In this section, we discuss how to classify a set of unseen instances $U^{\dagger}$ using optimization problem (\ref{eq:general_optimization}). 

\subsection{Binary classification}
 
 In this case, we need to assign a membership degree of instances from $U^{\dagger}$ to a set $A$, where $A$ refers to one of the classes. Solving optimization procedure (\ref{eq:general_optimization}) does not return an explicit prediction function $f: U \rightarrow [0,1]$ which would assign a membership degree to any new and unseen instance from $U^{\dagger}$. However, the membership degree of any new instance has to satisfy the constraints from (\ref{eq:general_optimization}).

Let $u^{\dagger} \in U^{\dagger}$. The aim is to estimate the membership degree $\hat{A}(u^{\dagger})$. Since unseen objects are represented with condition attributes, the relations $\Tilde{R}(u^{\dagger}, u)$ and $\Tilde{R}(u, u^{\dagger})$ can be calculated for all $u \in U$ and therefore, we assume that they are known. From the constraints of (\ref{eq:general_optimization}), we conclude that the conditions:
$$
\forall u \in U; \, T(\widetilde{R}(u^{\dagger}, u), \hat{A}(u)) \leq \hat{A}(u^{\dagger}),
$$
and
$$
\forall u \in U; \, T(\widetilde{R}(u, u^{\dagger}), \hat{A}(u^{\dagger})) \leq \hat{A}(u) \Leftrightarrow \forall u \in U; \, \hat{A}(u^{\dagger}) \leq I(\widetilde{R}(u, u^{\dagger}), \hat{A}(u)),
$$
have to be satisfied. The previous conditions can be rewritten as:

\begin{equation}
\label{eq:bounds}
    \max_{u \in U} T(\widetilde{R}(u^{\dagger}, u), \hat{A}(u)) \leq \hat{A}(u^{\dagger}) \leq \min_{u \in U} I(\widetilde{R}(u, u^{\dagger}), \hat{A}(u)).
\end{equation}

Expression (\ref{eq:bounds}) determines a lower and an upper bound for membership degree $\hat{A}(u^{\dagger})$ which forms an interval to which the degree should belong. First, we have to show that the interval is well defined.

\begin{proposition}
\label{prop:well_defined}
For any $u^{\dagger} \in U^{\dagger}$, it holds that 
$$
\max_{u \in U} T(\widetilde{R}(u^{\dagger}, u), \hat{A}(u)) \leq \min_{u \in U} I(\widetilde{R}(u, u^{\dagger}), \hat{A}(u)).
$$
\end{proposition}
\begin{proof}
An equivalent formulation of the demonstrandum is:
\begin{equation}
\label{eq:bounds_well_defined2}
\forall u,v \in U, \, T(\widetilde{R}(u^{\dagger}, u), \hat{A}(u)) \leq  I(\widetilde{R}(v, u^{\dagger}), \hat{A}(v)).
\end{equation}
Using granular representability, $T$-transitivity and associativity of $T$, we have that
\begin{align*}
    \hat{A}(v) &\geq T(\widetilde{R}(v, u), \hat{A}(u)) \\
        &\geq T(T(\widetilde{R}(v, u^{\dagger}), \widetilde{R}(u^{\dagger}, u)), \hat{A}(u)) \\
        &= T(\widetilde{R}(v, u^{\dagger}), T(\widetilde{R}(u^{\dagger}, u), \hat{A}(u)).
\end{align*}
The latter is equivalent to the formulation of the proposition due to the residuation property. 
\end{proof}
Since the interval is well defined, the next step is to properly aggregate the lower and upper bounds into one value. Denote
\begin{align}
\label{eq:bounds_definitions}
\underline{\hat{A}}(u^{\dagger}) =\max_{u \in U} T(\widetilde{R}(u^{\dagger}, u), \hat{A}(u)), \quad
\overline{\hat{A}}(u^{\dagger}) = \min_{u \in U} I(\widetilde{R}(u, u^{\dagger}), \hat{A}(u)).
\end{align}
Since $\hat{A}(u^{\dagger})$ represents the predicted membership degree of $u^{\dagger}$ to $A$, then $N(\hat{A}(u^{\dagger}))$ represents the membership degree to $coA$. Let $\mathbbm{A}$ be an averaging operator. We construct the prediction of the membership degree of $u^{\dagger} \in A$ as 
\begin{equation}
\label{eq:averaging_prediction}
\hat{A}(u^{\dagger}) = \mathbbm{A}(\underline{\hat{A}}(u^{\dagger}), \overline{\hat{A}}(u^{\dagger})).
\end{equation}

The next question is how to construct the averaging operator $\mathbbm{A}$. If (\ref{eq:averaging_prediction}) holds, then some sort of duality should also hold, i.e.,
\begin{equation}
\label{eq:dual_averaging}
N(\hat{A}(u^{\dagger})) = \mathbbm{A}(co\underline{\hat{A}}(u^{\dagger}), co\overline{\hat{A}}(u^{\dagger})),
\end{equation}
where
\begin{align*}
&co\underline{\hat{A}}(u^{\dagger}) = \max_{u \in U} T(\widetilde{R}(u, u^{\dagger}), N(\hat{A}(u))), \\ 
&co\overline{\hat{A}}(u^{\dagger}) = \min_{u \in U} I(\widetilde{R}(u^{\dagger}, u), N(\hat{A}(u))).
\end{align*}
We have the following result.
\begin{proposition}
\label{prop:prediction_duality}
For every $u^{\dagger} \in U^{\dagger}$, it holds that
\begin{equation*}
co\overline{\hat{A}}(u^{\dagger}) = N(\underline{\hat{A}}(u^{\dagger})), \quad 
co\underline{\hat{A}}(u^{\dagger}) = N(\overline{\hat{A}}(u^{\dagger})).
\end{equation*}
\end{proposition}
\begin{proof}
For the left equality, we have that
\begin{align*}
    N(\underline{\hat{A}}(u^{\dagger})) &= N( \max_{u \in U} T(\widetilde{R}(u^{\dagger}, u), \hat{A}(u))) \\
    &= \min_{u \in U} N(T(\widetilde{R}(u^{\dagger}, u), \hat{A}(u))) \\
    &= \min_{u \in U} I(\widetilde{R}(u^{\dagger}, u), N(\hat{A}(u))) = co\overline{\hat{A}}(u^{\dagger}).
\end{align*}
The third equality holds from (\ref{eq:implicator_negator_property2}).
For the right equality, we have that
\begin{align*}
    N(\overline{\hat{A}}(u^{\dagger})) &= N( \min_{u \in U} I(\widetilde{R}(u, u^{\dagger}), \hat{A}(u))) \\
    &= \max_{u \in U} N(I(\widetilde{R}(u, u^{\dagger}), \hat{A}(u))) \\
    &= \max_{u \in U} T(\widetilde{R}(u, u^{\dagger}), N(\hat{A}(u))) = co\underline{\hat{A}}(u^{\dagger}).
\end{align*}
The third equality holds from (\ref{eq:implicator_negator_property4}).
\end{proof}
Following Proposition \ref{prop:prediction_duality}, we conclude that for an aggregation operator $\mathbbm{A}$, it should hold that $N(\hat{A}(u^{\dagger})) = \mathbbm{A}(N(\hat{A}(u^{\dagger})), N(\hat{A}(u^{\dagger})))$, i.e., it is sufficient that $\mathbbm{A}$ is $N$-invariant. 

For an involutive negator $N$, let $\varphi_N$ be an isomorphism between $N$ and $N_s$, i.e., $N=\varphi_N^{-1}(N_s(\varphi_N))$. We define an averaging operator:
\begin{equation}
\label{eq:new_aggregation}    
\mathbbm{A}_N(x,y) = \varphi^{-1}_{N} \left ( \frac{\varphi_N(x) + \varphi_N(y)}{2} \right ).
\end{equation}
It is easily verifiable that $\mathbbm{A}_N$ is indeed $N$-invariant. 

Therefore, we predict the membership degree of $u^{\dagger}$ as $\hat{A}(u^{\dagger}) = \mathbbm{A}_N(\underline{\hat{A}}(u^{\dagger}), \overline{\hat{A}}(u^{\dagger})).
$
After obtaining the predicted membership degree, we have to defuzzify it, i.e., to obtain a crisp binary prediction. We return prediction $1$, i.e., $u^{\dagger}$ belongs to decision $A$ if $\hat{A}(u^{\dagger}) > N(\hat{A}(u^{\dagger}))$, and prediction 0 otherwise. Please note that when $\hat{A}(u^{\dagger}) = N(\hat{A}(u^{\dagger}))$, we have a tie and any prediction can be assigned. However, we will assign prediction $0$ in order to keep the deterministic nature of the prediction model. The condition
$\hat{A}(u^{\dagger}) > N(\hat{A}(u^{\dagger}))$ can be rewritten as
\begin{align*}
    \hat{A}(u^{\dagger}) > N(\hat{A}(u^{\dagger})) &\Leftrightarrow \hat{A}(u^{\dagger}) > \varphi_N^{-1}(1 - \varphi_N(\hat{A}(u^{\dagger}))) \\
    &\Leftrightarrow \varphi_N(\hat{A}(u^{\dagger})) > 1 - \varphi_N(\hat{A}(u^{\dagger})) \\
    &\Leftrightarrow \varphi_N(\hat{A}(u^{\dagger})) > \frac{1}{2} \Leftrightarrow \hat{A}(u^{\dagger}) > \varphi^{-1}_N  (0.5).
\end{align*}
We obtain that value $\varphi^{-1}_N  (0.5)$ is the threshold that determines the decision. 

In order to speed up the calculation, we can use the following 
proposition. 
\begin{proposition}
\label{prop:classification_redefinition}
In the binary classification case, it holds that
\begin{align}
    \underline{\hat{A}}(u^{\dagger}) =\max_{u \in \Bar{A}} T(\widetilde{R}(u^{\dagger}, u), \hat{A}(u)), \quad
    \overline{\hat{A}}(u^{\dagger}) = \min_{u \in co\Bar{A}} I(\widetilde{R}(u, u^{\dagger}), \hat{A}(u)).
\end{align}
\end{proposition}
\begin{proof}
An equivalent formulation of the demonstrandum, which holds from the granularity property, is
\begin{align}
\label{eq:adjacent_granules2}
\exists u \in \Bar{A}; \, \underline{\hat{A}}(u^{\dagger}) = T(\widetilde{R}(u^{\dagger}, u), \hat{A}(u)), \quad \exists u \in co\Bar{A}; \, \overline{\hat{A}}(u^{\dagger}) = I(\widetilde{R}(u, u^{\dagger}), \hat{A}(u)).
\end{align}
We prove the first equality from (\ref{eq:adjacent_granules2}). If the maximum from (\ref{eq:bounds_definitions}) is achieved for some $u \in \Bar{A}$, the equality is true. Otherwise, we assume that for some $u \in co\Bar{A}$, it holds that
$$
\underline{\hat{A}}(u^{\dagger}) =T(\widetilde{R}(u^{\dagger}, u), \hat{A}(u)).
$$
From Proposition \ref{eq:adjacent_granules}, there exists some $v \in \Bar{A}$ such that $\hat{A}(u) = T(\widetilde{R}(u, v), \Bar{A}(v))$. We have that
\begin{align*}
    \underline{\hat{A}}(u^{\dagger}) &=T(\widetilde{R}(u^{\dagger}, u), \hat{A}(u)) \\
        &= T(\widetilde{R}(u^{\dagger}, u), T(\widetilde{R}(u, v), \Bar{A}(v))) \\
        &= T(T(\widetilde{R}(u^{\dagger}, u), \widetilde{R}(u, v)), \Bar{A}(v)) \\
        &\leq T(\widetilde{R}(u^{\dagger}, v), \Bar{A}(v)).
\end{align*}
The inequality holds from the $T$-transitivity property. The opposite inequality holds from the granularity property.

For the second inequality from (\ref{eq:adjacent_granules2}), assume that the minimum from (\ref{eq:bounds_definitions}) is achieved for some $u \in \Bar{A}$. It holds that
$$
\overline{\hat{A}}(u^{\dagger}) =I(\widetilde{R}(u, u^{\dagger}), \hat{A}(u)).
$$

From Proposition \ref{eq:adjacent_granules}, we find that there exists some $v \in co\Bar{A}$ such that $\hat{A}(u) = I(\widetilde{R}(v, u), \Bar{A}(v))$. We have that
\begin{align*}
    \overline{\hat{A}}(u^{\dagger}) &=I(\widetilde{R}(u, u^{\dagger}), \hat{A}(u)) \\
        &= I(\widetilde{R}(u, u^{\dagger}), I(\widetilde{R}(v, u), \Bar{A}(v))) \\
        &= I(T(\widetilde{R}(v, u), \widetilde{R}(u, u^{\dagger})), \Bar{A}(v)) \\
        &\geq I(\widetilde{R}(v, u^{\dagger}), \Bar{A}(v)).
\end{align*}
The third equality holds because of (\ref{eq:t_norm_implicator_property}), while the inequality follows from the $T$-transitivity property. The opposite inequality holds from the granularity property which completes the proof.
\end{proof}

\subsection{Multi-class classification}

In the multi-class classification case, we assume that we have $K$ decision classes denoted with $A_1, \dots, A_K$. Let $\Bar{A}_1, \dots, \Bar{A}_K$ be observed decision classes from $U$ that are pairwise disjoint and which union is equal to $U$. Then, for $u, v \in U$, relation $S$ from \ref{eq:disjoint_optimization} is defined as $S(u, v) = 1$ if $\exists k \in \{1, \dots, K\}$ such that $u \in \Bar{A}_k \land v \in \Bar{A}_k$ and $S(u,v) = 0$ otherwise. Let $\beta_u$ be a solution of (\ref{eq:disjoint_optimization}) with such $S$. We want to estimate the membership values of an object $u \in U$ in class $k$ denoted with $\hat{A}_k(u)$. From the interpretation of $\beta_u$, we have that $\hat{A}_k(u) = \beta_u$ if $u \in \Bar{A}_k$. If $u \notin  \Bar{A}_k$, we use the right expression of (\ref{eq:adjacent_granules}) to obtain $\hat{A}_k(u)$ i.e.,
$$
\hat{A}_k(u) = \max_{v \in \Bar{A}_k} T(\widetilde{R}(u, v), \hat{A}_k(v)) = \max_{v \in \Bar{A}_k} T(\widetilde{R}(u, v), \beta_v).
$$

Using the same reasoning as in the binary classification, for $k \in \{1, \dots, K\}$ and using Proposition \ref{prop:classification_redefinition}, a lower and upper bounds of a membership degree of $u^{\dagger}$ in $A_k$ is defined as
\begin{equation*}
\underline{\hat{A}}_k(u^{\dagger}) =\max_{u \in \Bar{A}_k} T(\widetilde{R}(u^{\dagger}, u), \hat{A}_k(u)), \quad 
    \overline{\hat{A}}_k(u^{\dagger}) = \min_{u \in U-\Bar{A}_k} I(\widetilde{R}(u, u^{\dagger})), \hat{A}_k(u)),
\end{equation*}
while the prediction of the membership degree is obtained using averaging operator (\ref{eq:new_aggregation}). The decision class is then determined using formula:
$$
decision(u^{\dagger}) = \argmax_{k \in \{1, \dots, K\}} \mathbbm{A}(\underline{\hat{A}}_k(u^{\dagger}), \overline{\hat{A}}_k(u^{\dagger})).
$$

\subsection{Soft minimum and maximum}
From (\ref{eq:bounds}) and (\ref{eq:averaging_prediction}) we observe that the prediction of the membership degree of $u^{\dagger}$ is obtained based on the extreme values, i.e., the maximum from the left inequality and the minimum from the right inequality. In order to utilize more non-extreme values, we replace $\max$ and $\min$ with OWA operators. One motivation in using OWA operators and softening minimum and maximum in general is to reduce the influence of possible outliers in the dataset. The extreme values may correspond to outliers which makes the predicted membership degree unreliable. Hence, we would like to explore if using OWA operators will increase the performance of the classification model. 

For some weights $W_L$ and $W_U$ that correspond to soft $\min$ and $\max$ operators respectively, we have the following definitions:

\begin{align*}
\label{eq:bounds_definitions}
\underline{\hat{A}}^{W_U}(u^{\dagger}) =OWA_{W_U}\{T(\widetilde{R}(u^{\dagger}, u), \hat{A}(u)) ; u \in U \}, \\
\overline{\hat{A}}^{W_L}(u^{\dagger}) = OWA_{W_L}\{I(\widetilde{R}(u, u^{\dagger}), \hat{A}(u)) ; u \in U \},
\end{align*}
while the estimated membership is obtained in the same way as in (\ref{eq:averaging_prediction}).
From the definition of OWA, for all $u^{\dagger} \in U^{\dagger}$ it holds that
$$
\underline{\hat{A}}^{W_U}(u^{\dagger}) \leq \underline{\hat{A}}(u^{\dagger}), \quad \overline{\hat{A}}^{W_L}(u^{\dagger}) \geq  \overline{\hat{A}}(u^{\dagger}),
$$
which further implies that $\underline{\hat{A}}^{W_U}(u^{\dagger}) \leq \overline{\hat{A}}^{W_L}(u^{\dagger})$, i.e., the bounds are well-defined. 

The next question is if the duality expressed in analogous form as in (\ref{eq:dual_averaging}) and for $N$-invariant averaging operator $\mathbbm{A}$ will hold for
\begin{align*}
co\underline{\hat{A}}^{W_U}(u^{\dagger}) = OWA_{W_U}\{ T(\widetilde{R}(u, u^{\dagger}), N(\hat{A}(u)));  u \in U\}, \\ 
co\overline{\hat{A}}^{W_L}(u^{\dagger}) = OWA_{W_L} \{ I(\widetilde{R}(u^{\dagger}, u), N(\hat{A}(u))); u \in U \}.
\end{align*}
If we consider the proof of Proposition \ref{prop:prediction_duality}, we conclude that the answer to the previous question depends on whether OWA operators and negator $N$ are interchangeable. This is not always the case, but we do have the following proposition.
\begin{proposition}
\label{prop:prediction_duality_owa}
\end{proposition}
Let $(T, I, N)$ be a residual triplet for which $N$ is the standard negator and let $W_U$ and $W_L$ be complementary vectors of weights. Then, it holds that
$$
co\overline{\hat{A}}^{W_L}(u^{\dagger}) = N(\underline{\hat{A}}^{W_U}(u^{\dagger})), \quad 
co\underline{\hat{A}}^{W_U}(u^{\dagger}) = N(\overline{\hat{A}}^{W_L}(u^{\dagger})).
$$
\begin{proof}
We prove the first equality, while the second one holds by analogy.
Let $u_1, \dots, u_n$ be an ordering of instances from $U$ such that
$$
T(\widetilde{R}(u^{\dagger}, u_1), \hat{A}(u_1)) \geq \dots \geq T(\widetilde{R}(u^{\dagger}, u_n), \hat{A}(u_n)).
$$
Applying negator $N$ to the previous inequalities and using the fact that $N$ is decreasing, together with property (\ref{eq:implicator_negator_property2}), we have that
$$
I(\widetilde{R}(u^{\dagger}, u_1), N(\hat{A}(u_1))) \leq \dots \leq I(\widetilde{R}(u^{\dagger}, u_n), N(\hat{A}(u_n))).
$$
Also,
\begin{align*}
    N(\underline{\hat{A}}^{W_U}(u^{\dagger})) &= 1 -  \sum_{u =1}^{n} (W_U)_i \cdot T(\widetilde{R}(u^{\dagger}, u_i), \hat{A}(u_i))) \\
    &= \sum_{u =1}^{n} (W_U)_i \cdot (1 -  T(\widetilde{R}(u^{\dagger}, u_i), \hat{A}(u_i))) \\
    &= \sum_{u =1}^{n} (W_U)_i \cdot  I(\widetilde{R}(u^{\dagger}, u_i), N(\hat{A}(u_i))) \\
    &= \sum_{u =1}^{n} (W_L)_i \cdot  I(\widetilde{R}(u^{\dagger}, u_{n - i + 1}), N(\hat{A}(u_{n - i + 1}))) = co\overline{\hat{A}}^{W_L}(u^{\dagger}). 
\end{align*}
The third equality holds form property (\ref{eq:implicator_negator_property2}) and the fact that $N$ is the standard negator. In the fourth equality, we replaced indices $i$ with indices $n - i + 1$ and applied the complementarity of $W_U$ and $W_L$.
\end{proof}
Proposition \ref{prop:prediction_duality_owa} states that if $\mathbbm{A}$ is $N$-invariant for $N$ the standard negator, the duality analogous to (\ref{eq:dual_averaging}) holds. An example of such an averaging operator is the arithmetic mean.

\section{Approximate calculation}
\label{sec:approx_calculation}
In this section, we discuss how to speed up the calculation of the optimization problems (\ref{eq:general_optimization}) and (\ref{eq:disjoint_optimization}) at the cost of the precision of the obtained granular approximations. In other words, while the time performance will be improved, the calculated approximations will slightly deviate from the granular approximations obtained using (\ref{eq:general_optimization}) and (\ref{eq:disjoint_optimization}).

We claim that the highest influence in calculating granular approximations is made by the most similar instances. For fixed $u \in U$, in the set of constraints $T(\widetilde{R}(u,v), \hat{A(v)}) \leq \hat{A}(u), v \in U$, we assert that if the constraints with higher $\widetilde{R}(u,v)$ are satisfied, then it is highly likely that the constraints with lower $\widetilde{R}(u,v)$ will also be satisfied.
Let $M$ be a positive integer and $u \in U$. We denote with $v_{(1)}, \dots, v_{(M)}$ objects from $U$ such that $\{\widetilde{R}(u, v_{1}), \dots, \widetilde{R}(u, v_{M})\}$ are the $M$ largest values from the set of values $\{ \widetilde{R}(u,v); v\in U\}$.

Then the approximate formulation of problem (\ref{eq:general_optimization}) is:
\begin{equation}
\label{eq:general_optimization_approx}
\begin{aligned}
&\text{minimize}  && \displaystyle\sum_{u \in U}  L(\Bar{A}(u), \hat{A}(u)) \\
&\text{subject to}    && T( \widetilde{R}(u,v_{(i)}), \hat{A}(v_{(i)})) \leq \hat{A}(u), \quad   u \in U, i \in \{1, \dots, M \}\\
  &              &&0 \leq \hat{A}(u) \leq 1, \quad u \in U,
\end{aligned}
\end{equation}
while the approximate formulation of problem (\ref{eq:disjoint_optimization}) is
\begin{equation}
\label{eq:disjoint_optimization_approx}
\begin{aligned}
&\text{minimize}  && \displaystyle\sum_{u \in U} L(1, \beta_u) \\
&\text{subject to} && T(\beta_u, \beta_{v_{i}}) \leq I(\widetilde{R}(u,v_{i}), S(u,v_{i}))),  \quad   u \in U, i \in \{1,\dots, M \}\\
  &              &&0 \leq \beta_u \leq 1, \quad u \in U.
\end{aligned}
\end{equation}
Note that the instances $v_{(1)}, \dots, v_{(M)}$ are different for every $u$. We can observe that now instead of $|U|^2$ constraints, both problems have $M\cdot |U|$ constraints where usually $M << |U|$. With such reduction of the number of constraints, significant time savings can be achieved. 

In the left side of Figure \ref{fig:approximation_comparison}, we compare the granular approximations for different values of parameter $\gamma$, different loss functions and different similarity relations. Every row stands for one combination of a loss function and a similarity relation which is indicated in the individual titles of images. We express $M$ as a fraction of $|U|$ called the $nn$ parameter (short form of ``nearest neighbors"), i.e.  $nn = \frac{M}{|U|}$. For different $nn$ values from 0.5 to 0.01, we calculate the granular approximation and compare it with the case when $nn=1$, i.e., when all constraints are used. We calculate the absolute difference between two granular approximations and take the average as the measure of difference between two granular approximations (basically, we apply the MAE loss). We do this for every dataset from Section \ref{subsec:datasets} and then average the results.

\begin{figure}[H]
    \centering
    
    \includegraphics[width = .48\textwidth]{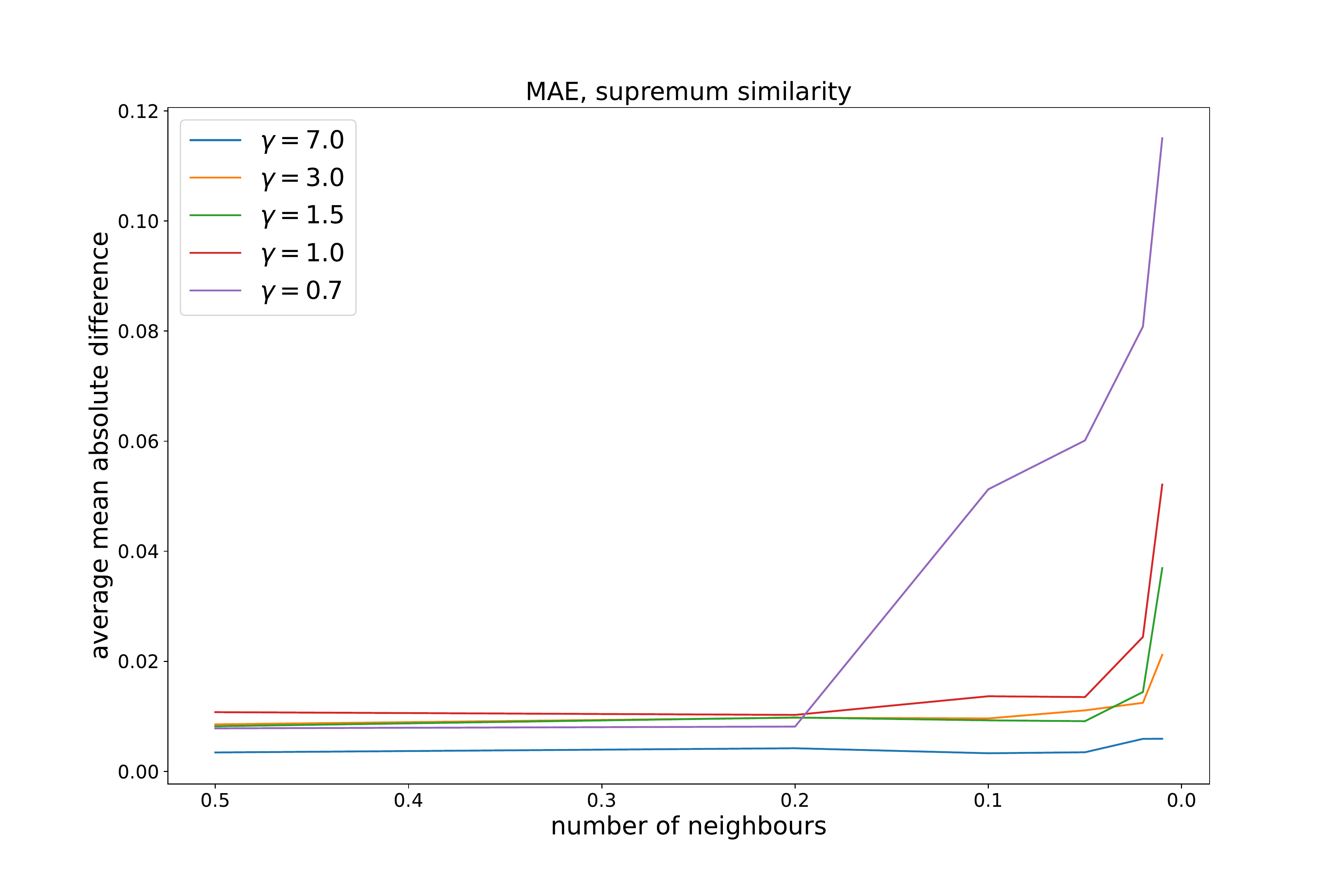}
    \includegraphics[width = .48\textwidth]{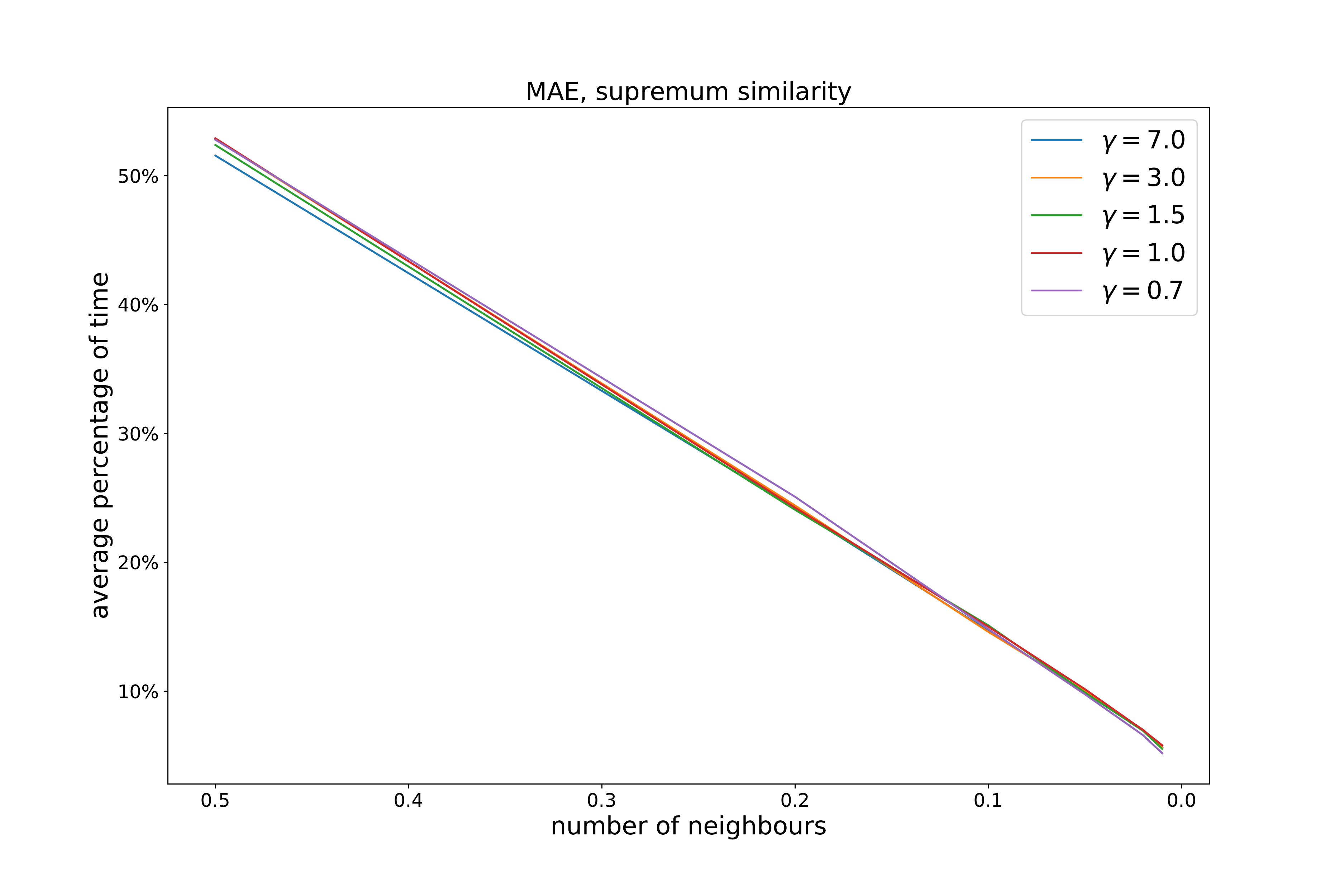}
    \includegraphics[width = .48\textwidth]{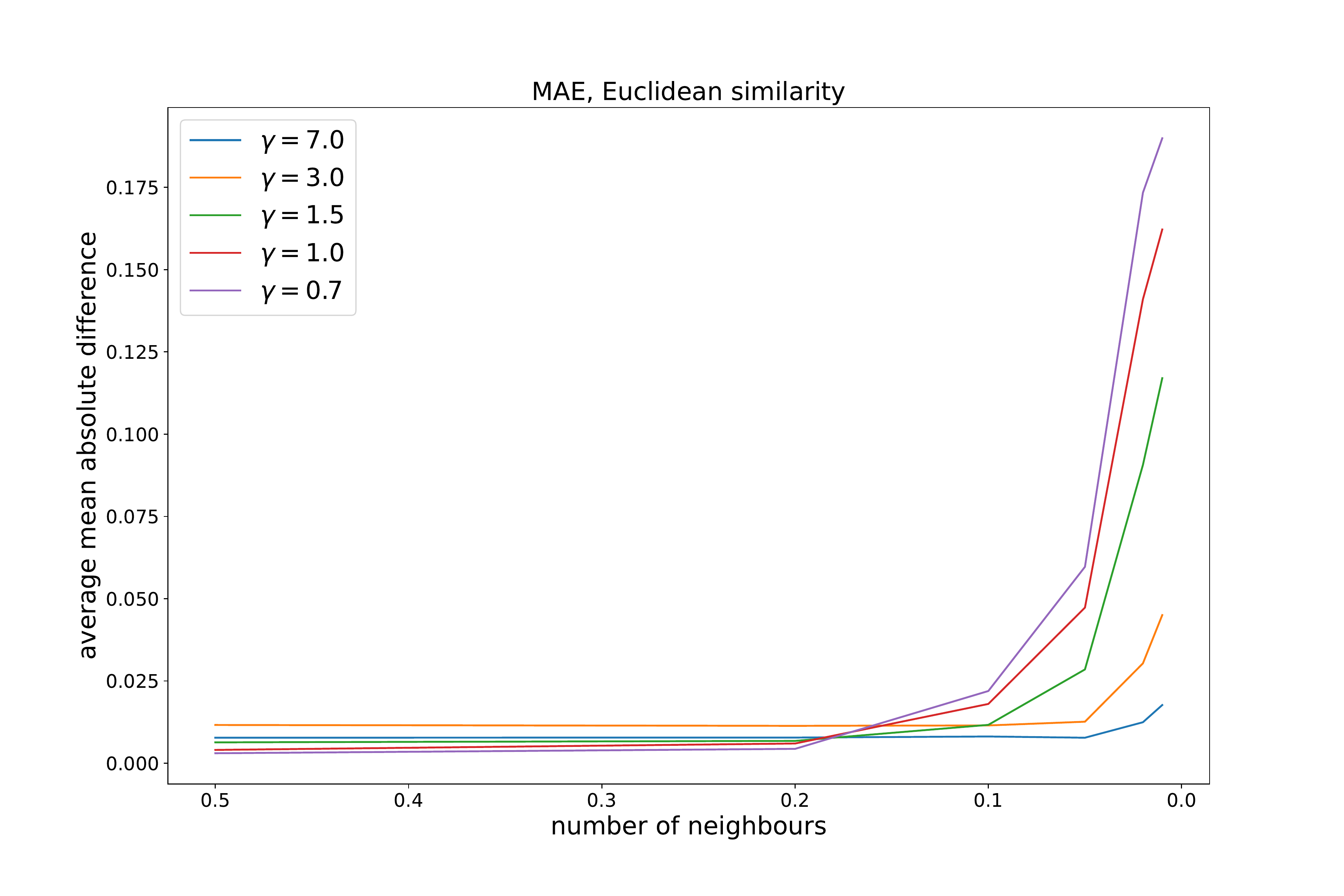}
    \includegraphics[width = .48\textwidth]{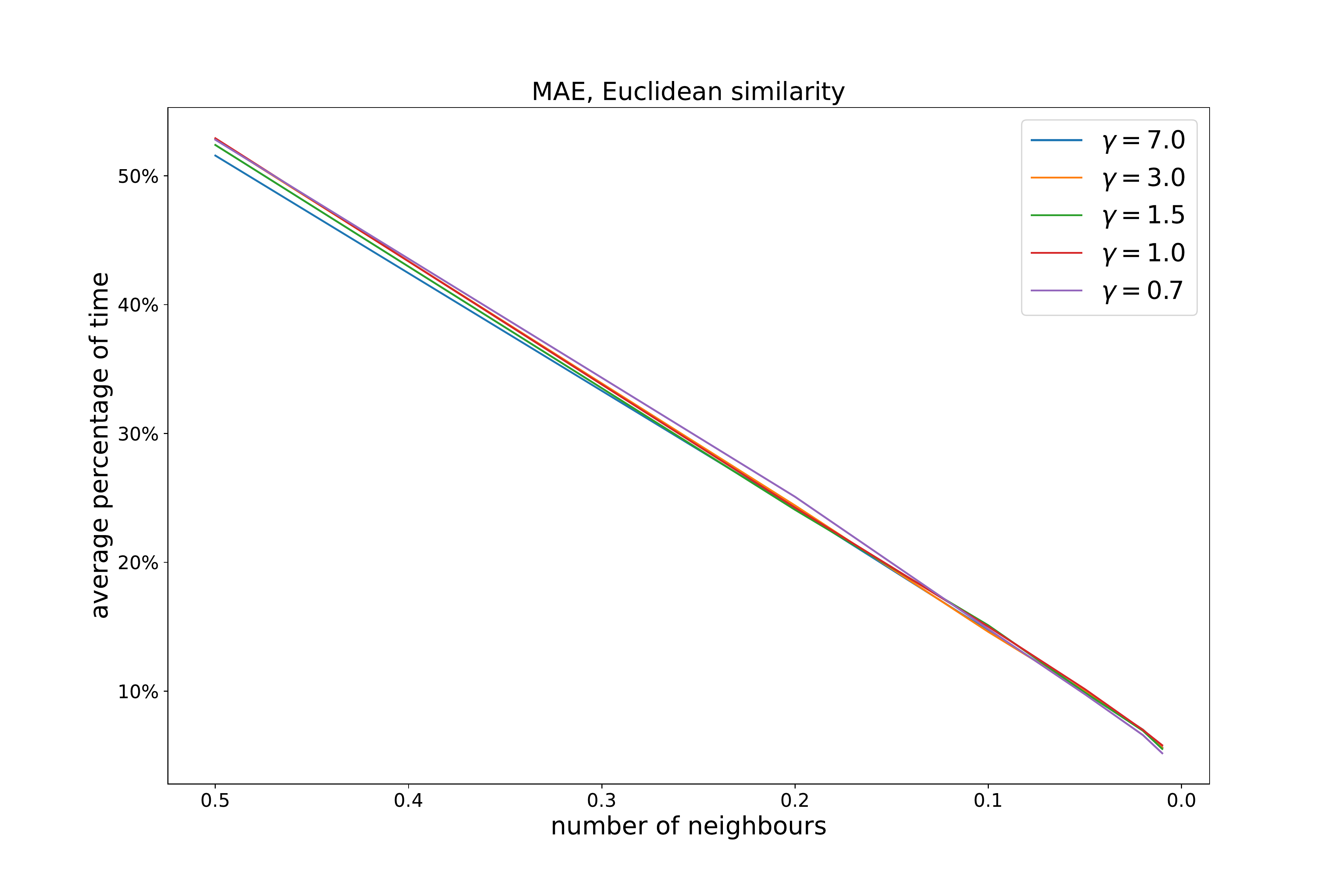}
    \includegraphics[width = .48\textwidth]{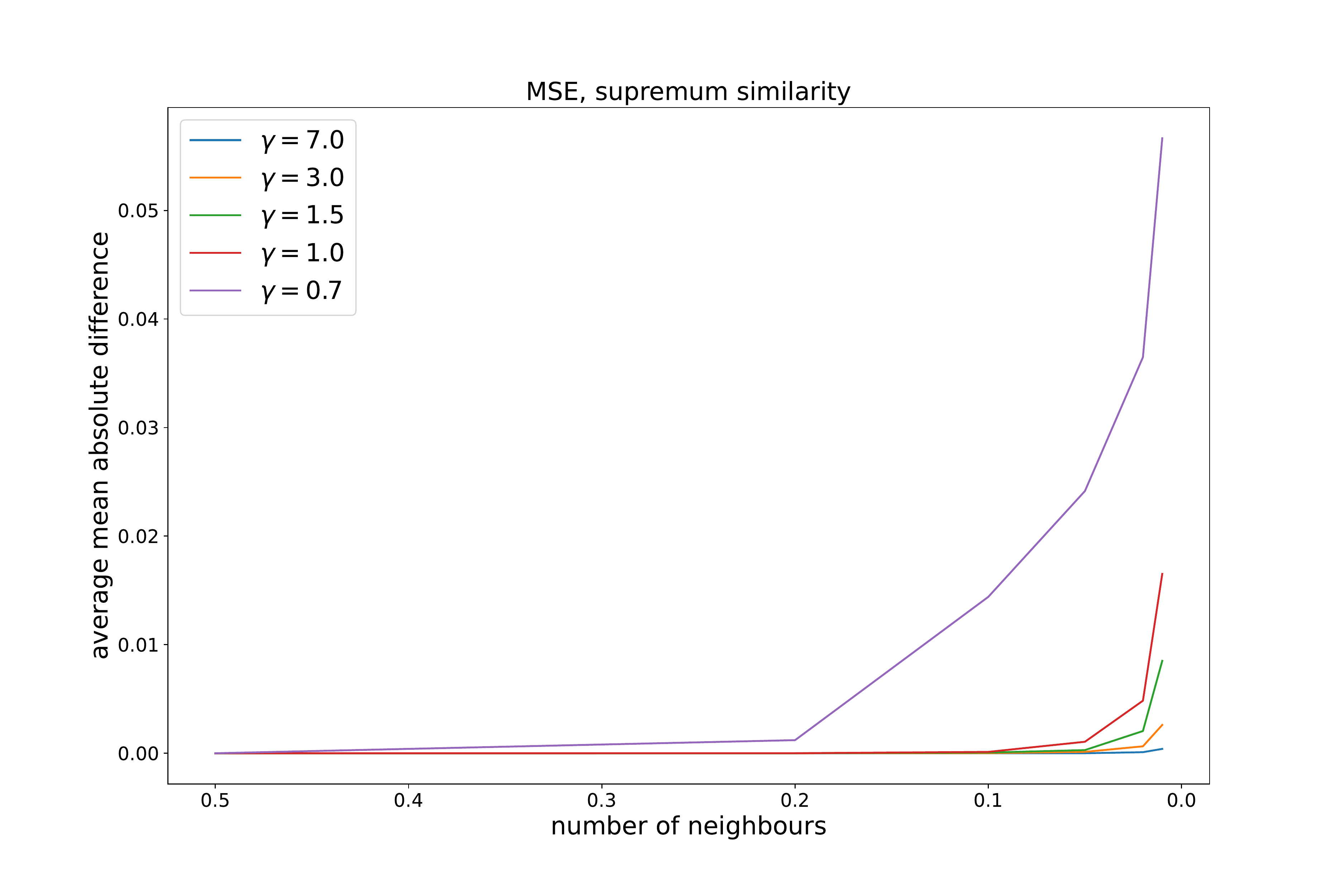}
    \includegraphics[width = .48\textwidth]{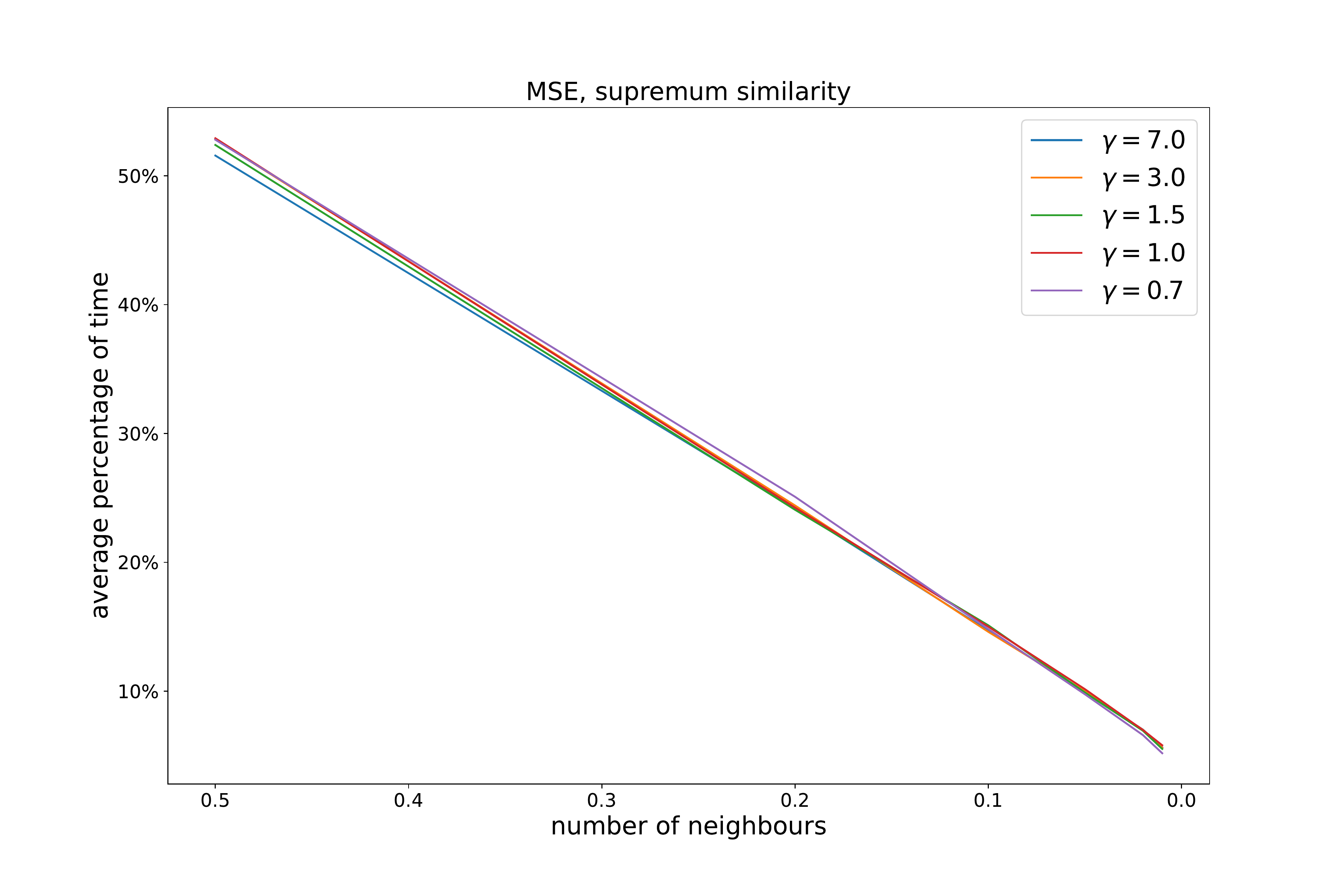}
    \includegraphics[width = .48\textwidth]{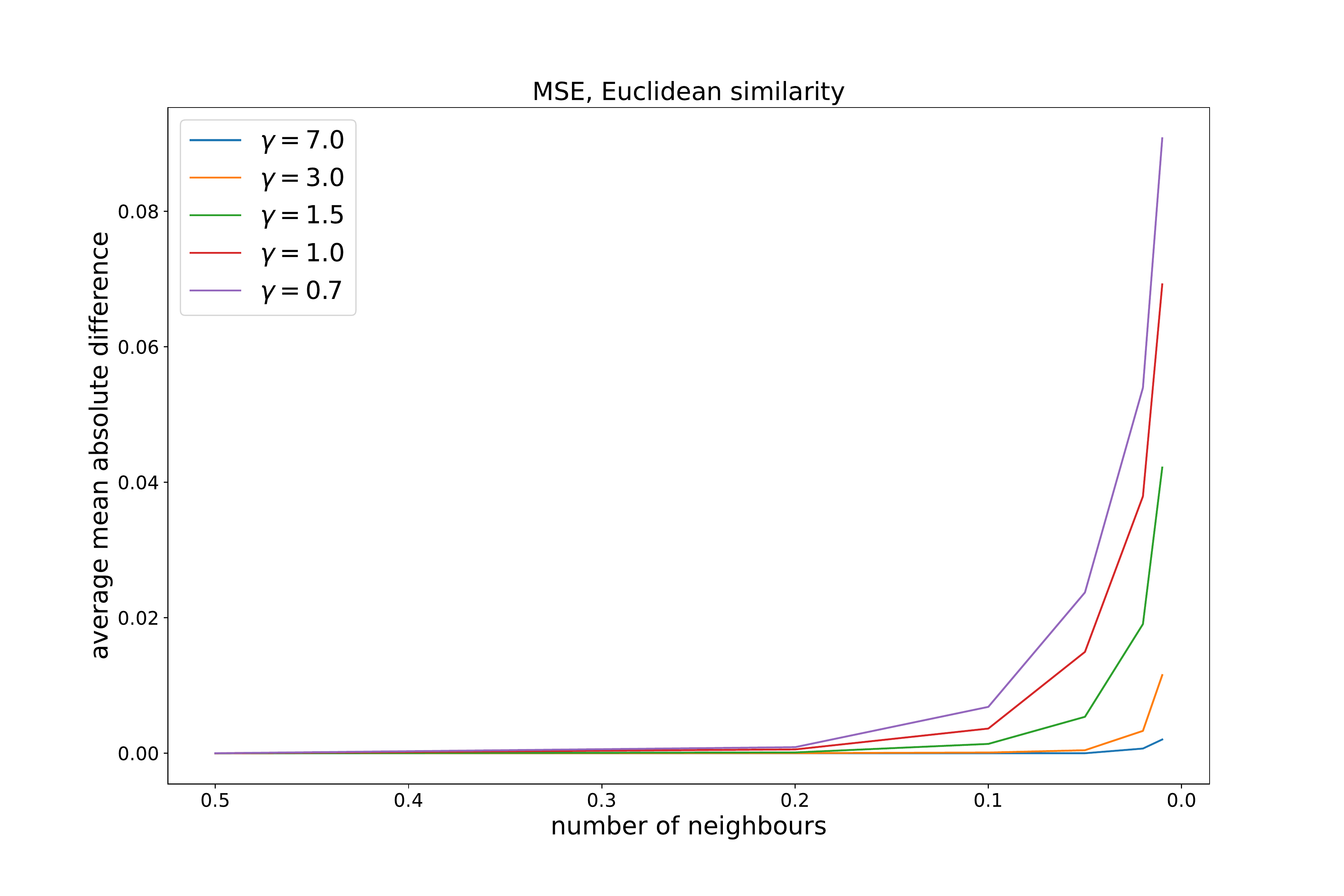}
    \includegraphics[width = .48\textwidth]{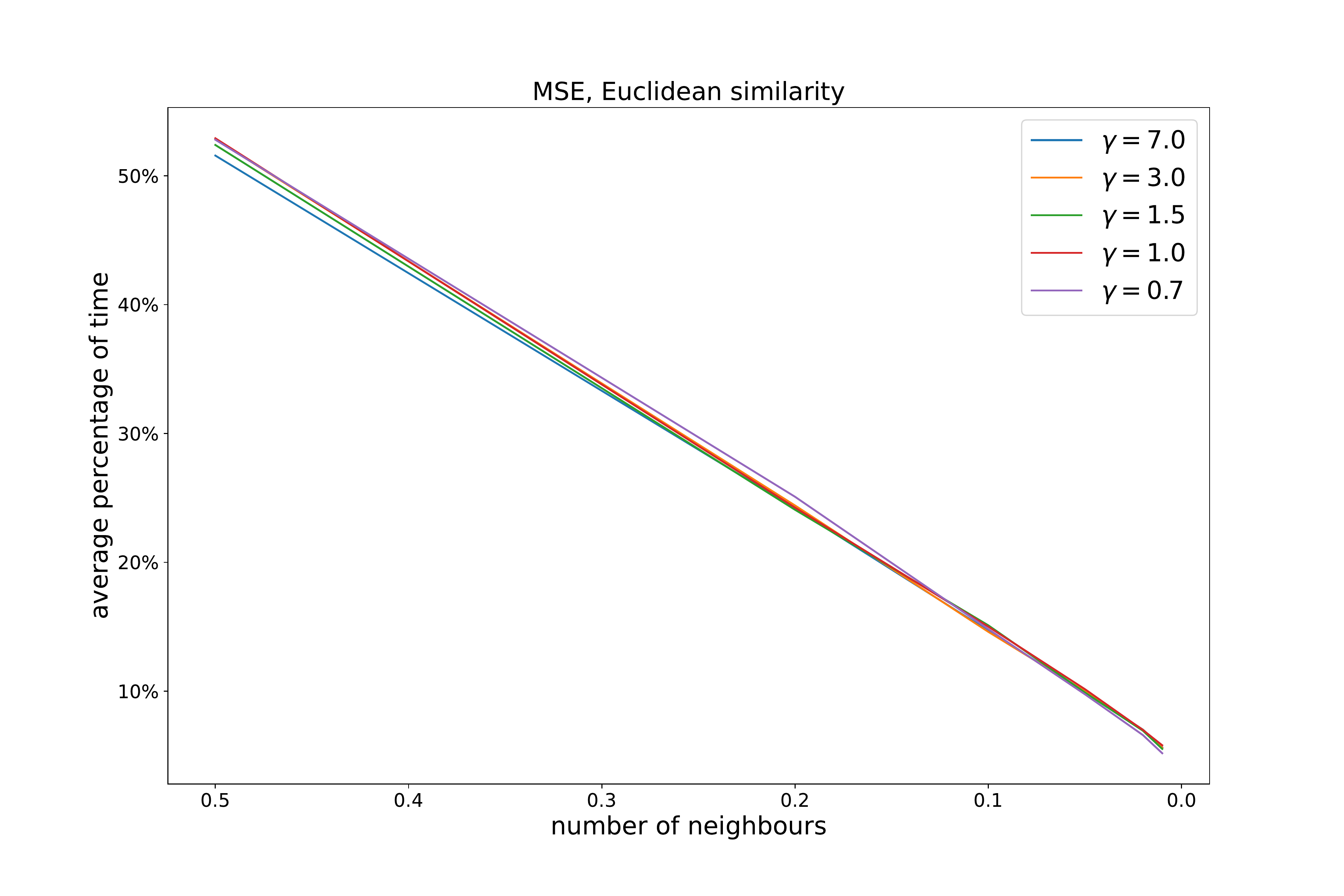}
    \caption{Approximate calculations of granular approximations - comparisons}
    \label{fig:approximation_comparison}
\end{figure}{}

In all images, we can clearly see that the difference is bigger when the $nn$ parameter is smaller. When reducing the $nn$ parameter, we observe that the difference is growing more rapidly for larger values of $\gamma$. Also, the difference is larger when MAE is used compared to MSE. In the right side of Figure \ref{fig:approximation_comparison}, we depict the times consumed to calculate the granular approximations as fractions of the time spent to calculate the granular approximation for $nn=1$. We can observe that the time needed for the calculation is decreasing linearly as a function of $nn$, for $nn$ becoming small. For example, if $nn=0.02$, i.e., if we use $0.02 \cdot |U|^2$ constraints, we save $98\%$ of time compared to the case where we use all constraints.

In the following section, we empirically test if reducing the number of constraints significantly affects the predictive performance of FGAC.

\section{Experiments}
\label{sec:experiments}
\subsection{Experimental setup}
In this section, we test the performance of FGAC and the models from Section \ref{subsec:models_to_compare_with} on data from Subsection \ref{subsec:datasets} together with the encoding of nominal attributes explained in the same subsection. We implemented FGAC in the PYTHON programming language \cite{van1995python}. In the current version, we used the Łukasiewicz $t$-norm and the corresponding IMTL triplet in order to evaluate the estimated membership degree (\ref{eq:averaging_prediction}). To solve optimization problems (\ref{eq:T_L_linear}) and (\ref{eq:T_L_quadratic}), we use the GUROBI solver \cite{gurobi} and its API for PYTHON. The code for the experiments is available on url: \url{https://github.com/markopalangetic/FGAC_experiments}.

The experiments are also implemented in the PYTHON programming language.
For every model, we select one hyperparameter which will be tuned, i.e., the hyperparameter for which the model performs best will be chosen. The interpretation of these hyperparameters is that they control the bias-variance trade-off, i.e., their tuning is used to balance between overfitting and underfitting.

For the kFRNN models, a parameter which controls the number of non-zero OWA weights is used. The approach is motivated by \cite{ramentol2015ifrowann}. In that case, only the last $k$ values of $W_L$ (first $k$ values of $W_U$) are non-zero. In these experiments, the non-zero values will be those introduced in Section \ref{subsec:owa} (additive, exponential, inverse additive).

For FGAC, $\gamma$ will be the hyperparameter that is tuned. We provide an example to illustrate that $\gamma$ is indeed a parameter that balances between bias and variance.  

\begin{figure}[H]
    \centering
    \includegraphics[width = .32\textwidth]{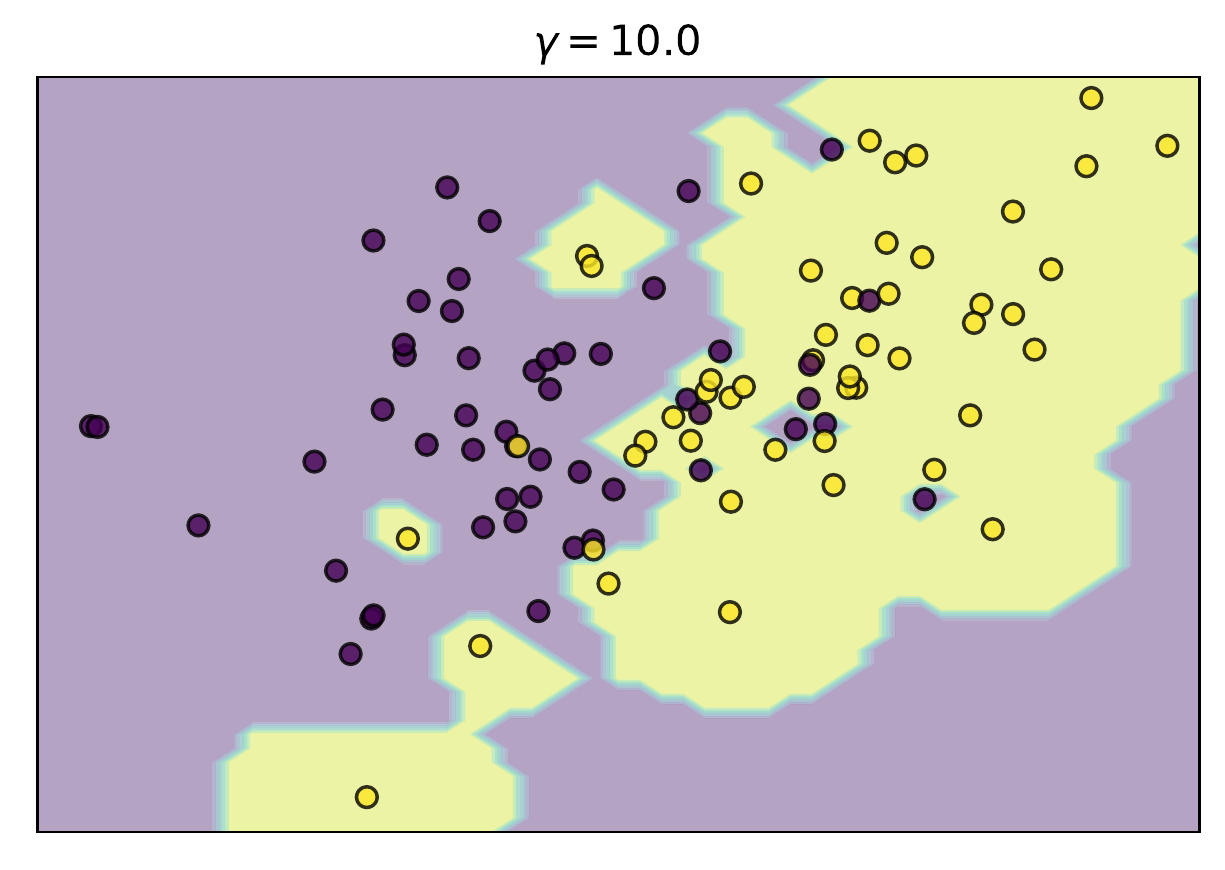}
    \includegraphics[width = .32\textwidth]{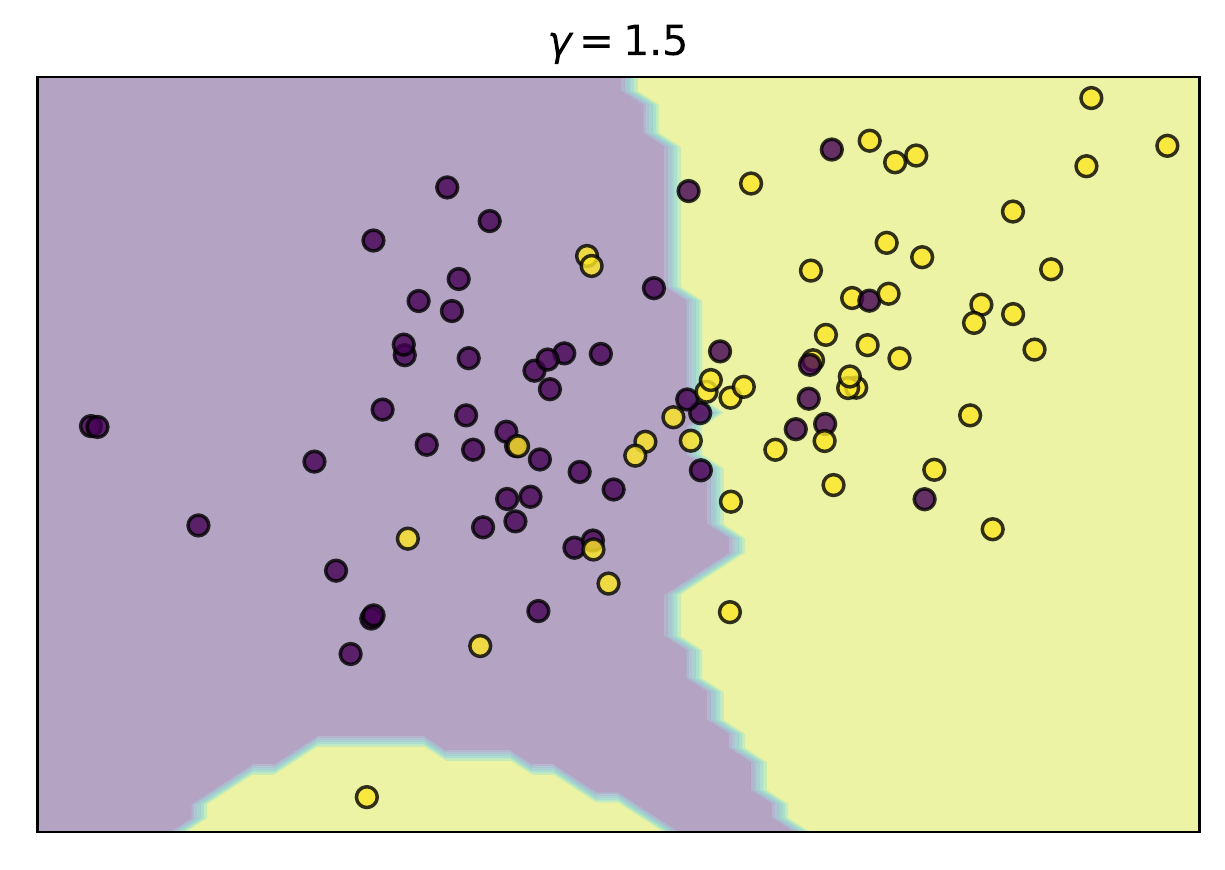}
    \includegraphics[width = .32\textwidth]{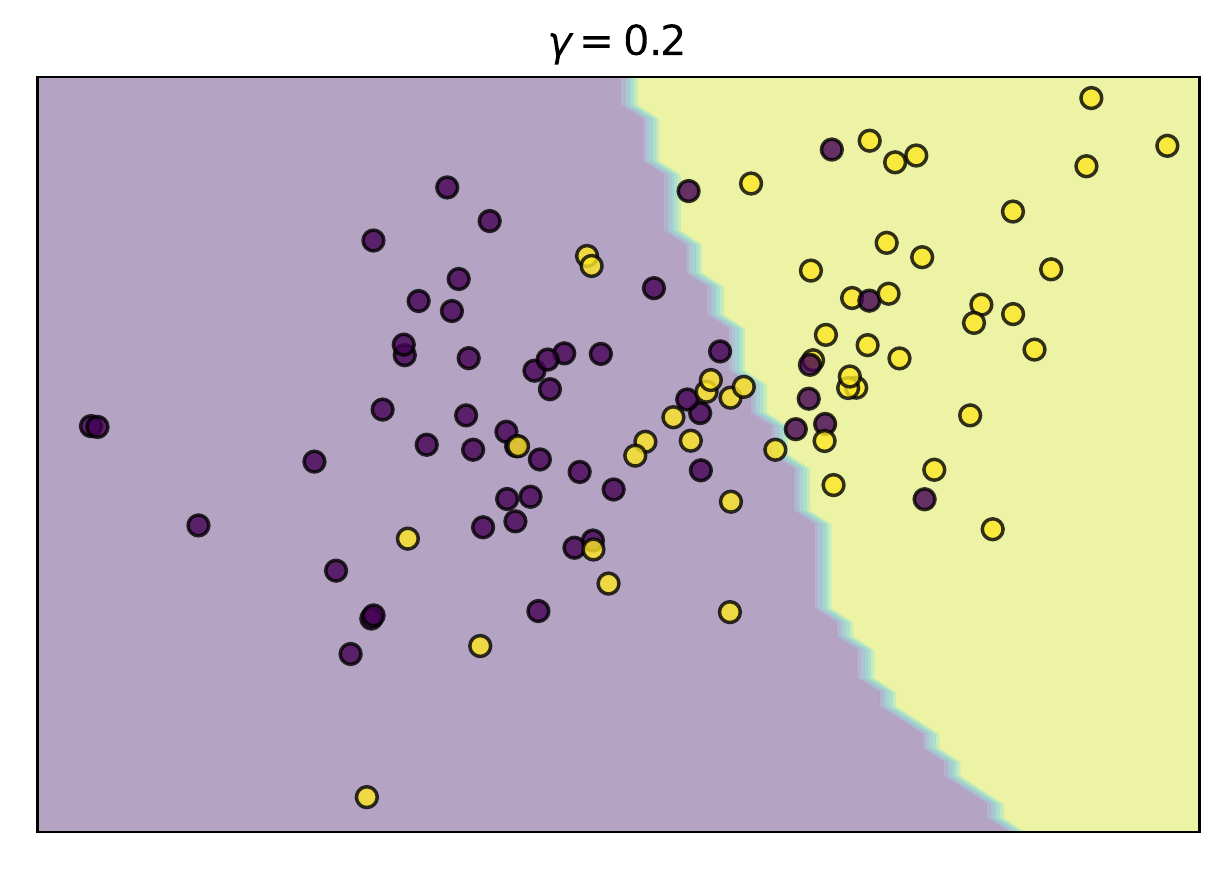}
    \caption{Illustrations of decision spaces for different $\gamma$}
    \label{fig:decision_spaces}
\end{figure}

In Figure \ref{fig:decision_spaces}, we generated 100 synthetic data instances for a binary classification problem to illustrate the decision areas for different values of parameter $\gamma$. The dataset was generated using the SCIKIT-LEARN package and the "make\_classification" function. The control of the random number generator is achieved with the command "rand\_state=10".

In the left image in Figure \ref{fig:decision_spaces}, we can see clear overfitting for $\gamma=10$, as an example of a high value, where the learning process is affected by the noise in data. As $\gamma$ decreases, we can see that the decision boundary (the line that separates the two decision classes) becomes smoother and simpler (middle image and right image in Figure \ref{fig:decision_spaces}) which indicates a less noise-affected learning process. For a very small $\gamma$ parameter (right image), we observe an even simpler decision line which may be a sign of underfitting and indicates that the model did not properly capture the relationship between the condition attributes and the decision attribute. 

\subsection{Comparison of the different versions of FGAC}

We first test if the approximate calculation of FGAC from Section \ref{sec:approx_calculation} affects the prediction performance. We evaluate the performance of FGAC for the $nn$ parameter equal to 1, .2 and .02, for both MAE and MSE loss functions given in (\ref{eq:mse}) and (\ref{eq:mae}) and for both Euclidean and supremum similarities. This gives in total $3\cdot2 \cdot 2 = 12$ different models. The models are applied to the data from Section \ref{subsec:datasets}.

Before executing the models, we apply random oversampling of the minority classes. We randomly sample instances from the minority classes and add copies of them to the dataset until all decision classes from the training set have an equal amount of instances that is equal to the size of the majority decision class. To implement the random oversampling, PYTHON package IMBLEARN and its "RandomOverSampler" class were used. To control the random number generator and to ensure the reproducibility of the results, we set "rand\_state=10".

In order to evaluate the performance of each model, 5-fold cross-validation is used to tune parameter $\gamma$ and to evaluate the performance. The performance evaluation metric used is balanced accuracy. After preliminary tests, we decided to tune $\gamma$ from the following 11 possible values: $\{0.5, 0.7, 0.8, 1, 1.5, 2, 3, 4, 5, 7, 10 \}$.
The fine-tuning and cross-validation are implemented in the SCIKIT-LEARN package and the "GridSearchCV" class where the controlling of the random number generator is achieved with "rand\_state=10". The results are shown in Table \ref{tab:results_approximation}. The names of the columns in the table are composed from the type of the loss function used ("mae" or "mse"), the type of the similarity relation ("supremum" or "Euclidean") and the $nn$ parameter ("nn2", "nn20", "nn100"). 

We first want to test if the performance of the models is different for $nn=1$ and for $nn=.02$. For that purpose, we use the Wilcoxon signed-rank test \cite{wilcoxon1992individual}. For every combination of a loss function and a similarity relation, we test if the performance of models is significantly different, where the null hypothesis is that they are not significantly different. We obtain the following p-values:
\begin{table}[H]
\begin{adjustbox}{width=\columnwidth,center}
\begin{tabular}{|l|l|l|l|l|}
\hline
           & mae\_supremum & mae\_Euclidean & mse\_supremum & mse\_Euclidean \\ \hline
$p$-value & 0.468         & 0.32           & 0.844         & 0.753          \\ \hline
\end{tabular}
\end{adjustbox}
\end{table}
We can observe that all $p$-values are significantly higher than the usual significance level 0.05 which means that, based on the provided evidence, we cannot conclude that the performance is significantly different (i.e., we cannot reject the null hypothesis). Based on these results, we conclude that in practical applications, using $nn=.02$ is sufficient for a desirable performance of FGAC. 

After discarding the cases $nn=1$ and $nn=.2$, we want to decide which model is the best among the remaining four. First we test if their performance is significantly different from each other. For that purpose, we use the non-parametric Friedman chi-squared test \cite{friedman1937use}. The null hypothesis of this test is that the performances of the models is indifferent. After running the test, we get that the $p$-value is of order $10^{-5}$, which means that we strongly reject the null hypothesis, i.e., the models are significantly different. The next step is to recognize the best model and to test if it is significantly better than the others. If we look at the average rankings of the models, we have the following:
\begin{table}[H]
\begin{adjustbox}{width=\columnwidth,center}
\begin{tabular}{|l|l|l|l|l|}
\hline
             & mae\_supremum & mae\_Euclidean & mse\_supremum & mse\_Euclidean \\ \hline
average rank & 3.5           & 2.11           & 2.78          & 1.61           \\ \hline
\end{tabular}
\end{adjustbox}
\end{table}
We observe that the model which uses mean squared error and Euclidean similarity has the best average ranking. We hypothesize that this model is the best performing one and we test if this is statistically significant. We use Holm post-hoc analysis \cite{holm1979simple} as well as its adaptation for comparing machine learning models from \cite{demvsar2006statistical}. Following the criticism of \cite{demvsar2006statistical} expressed in \cite{benavoli2016should}, we use the Wilcoxon test for the pairwise comparisons. After the Holm procedure is applied, the obtained final $p$-value is 0.099 which means that we cannot confidently claim that the best ranked model is significantly the best performing one. The final $p$-value in this case is obtained as a maximum of the adjusted $p$-values calculated during the Holm procedure.
However, due to its average ranking on the provided datasets, we continue using it as the representative version of FGAC in the comparison with the other ML models.

The next step is to test how the use of OWA operators affects the performance of FGAC. In Table \ref{tab:results_owa}, we list the results of OWA-based FGAC when $\min$ and $\max$ are replaced with OWA operators with weights from Section \ref{subsec:owa}. As before, the results are given for both MSE and MAE loss functions, as well as for both supremum and Euclidean similarities. First, we test if the 4 models for fixed OWA weights perform differently from each other using the Friedman test. We obtain the following results:
\begin{table}[H]
\centering
\begin{tabular}{|l|l|l|l|}
\hline
weight:   & add                               & exp                              & invadd                           \\ \hline
$p$-value: & $2.06 \cdot 10^{-5}$ & $2.1 \cdot 10^{-4}$ & $1.2 \cdot 10^{-3}$ \\ \hline
\end{tabular}
\end{table}
All $p$-values are very close to $0$ which means that the performances are indeed significantly different. If we calculate the average rankings, we find:

\begin{table}[H]
\begin{adjustbox}{width=\columnwidth,center}
\begin{tabular}{|l|l|l|l|l|}
\hline
       & mae\_supremum & mae\_Euclidean & mse\_supremum & mse\_Euclidean \\ \hline
add    & 3.17          & 2.33           & 3.11          & 1.39           \\ \hline
exp    & 3.39          & 2.44           & 2.61          & 1.56           \\ \hline
invadd & 3.33          & 2.33           & 2.56          & 1.78           \\ \hline
\end{tabular}
\end{adjustbox}
\end{table}
As in the non-OWA version of the FGAC, we observe that the best average ranking is achieved for the MSE loss function and the Euclidean similarity. As before, using post-hoc analysis we test if the performance of the best ranked model is significantly better than others. We obtain the following $p$-values:

\begin{table}[H]
\centering
\begin{tabular}{|l|l|l|l|}
\hline
            & add   & exp   & invadd \\ \hline
$p$-value: & 0.008 & 0.059 & 0.038  \\ \hline
\end{tabular}
\end{table}
From these $p$-values, we can conclude that for the additive and inverse additive weights, we can confidently say that the best ranking model performs better than the other models. For the exponential weights, the $p$-value is negligibly higher than the usual significance level (0.05). In any case, we will use the best ranking models as representatives of the particular OWA-weights in further comparisons. 

In the next step, we compare the performances of the chosen models for different OWA weights with the chosen FGAC model from before. To recall, we have 4 different models, 3 with different OWA weights (add, exp, invadd) where all 4 models use MSE loss and Euclidean similarity. After performing the Friedman test on their performances, we get a $p$-value equal to $0.293$ which can be considered high. In other words, we confidently claim that we do not have enough evidence to conclude that using OWA-operators instead of extrema operators will lead to different results. The reason for that may lie in the fact that the learning process is based on the constraints that use extrema instead of OWA operators. The latter are only used in the prediction phase and not in the learning phase. In other words, the learning phase is the key part and the OWA intervention during the prediction phase cannot improve the results.

For this reason, we exclude OWA-based FGAC from the further analysis.

\begin{landscape}
\begin{table}[]
    \centering
\begin{adjustbox}{width=\columnwidth}

\begin{tabular}{l|rrr|rrr|rrr|rrr}
\toprule
name &  
\makecell{fgac\_mae \\ supremum\_nn100} &
\makecell{fgac\_mae \\ supremum\_nn20} &
\makecell{fgac\_mae \\ supremum\_nn2} &
\makecell{fgac\_mae \\ Euclidean\_nn100} &
\makecell{fgac\_mae \\ Euclidean\_nn20} &
\makecell{fgac\_mae \\ Euclidean\_nn2} &
\makecell{fgac\_mse \\ supremum\_nn100} &
\makecell{fgac\_mse \\ supremum\_nn20} &
\makecell{fgac\_mse \\ supremum\_nn2} &
\makecell{fgac\_mse \\ Euclidean\_nn100} &
\makecell{fgac\_mse \\ Euclidean\_nn20} &
\makecell{fgac\_mse \\ Euclidean\_nn2} \\
\midrule
australian   &                                0.856 &                               0.855 &                              0.857 &                               0.863 &                              0.857 &                             0.847 &                                0.869 &                               0.868 &                              0.868 &                               0.867 &                              0.866 &                             0.862 \\
breast       &                                0.536 &                               0.536 &                              0.536 &                               0.661 &                              0.667 &                             0.663 &                                0.536 &                               0.536 &                              0.536 &                               0.682 &                              0.682 &                             0.666 \\
crx          &                                0.733 &                               0.736 &                              0.736 &                               0.820 &                              0.818 &                             0.827 &                                0.771 &                               0.776 &                              0.776 &                               0.827 &                              0.828 &                             0.825 \\
german       &                                0.509 &                               0.519 &                              0.519 &                               0.645 &                              0.664 &                             0.663 &                                0.540 &                               0.516 &                              0.516 &                               0.663 &                              0.663 &                             0.659 \\
saheart      &                                0.671 &                               0.671 &                              0.673 &                               0.679 &                              0.667 &                             0.663 &                                0.669 &                               0.672 &                              0.669 &                               0.674 &                              0.674 &                             0.686 \\
ionosphere   &                                0.929 &                               0.931 &                              0.925 &                               0.941 &                              0.933 &                             0.937 &                                0.929 &                               0.929 &                              0.929 &                               0.939 &                              0.939 &                             0.939 \\
mammographic &                                0.802 &                               0.799 &                              0.801 &                               0.808 &                              0.807 &                             0.812 &                                0.804 &                               0.803 &                              0.803 &                               0.804 &                              0.804 &                             0.804 \\
pima         &                                0.701 &                               0.704 &                              0.703 &                               0.719 &                              0.711 &                             0.720 &                                0.719 &                               0.719 &                              0.722 &                               0.731 &                              0.732 &                             0.734 \\
wisconsin    &                                0.955 &                               0.958 &                              0.954 &                               0.968 &                              0.968 &                             0.968 &                                0.966 &                               0.966 &                              0.966 &                               0.972 &                              0.972 &                             0.972 \\
vowel        &                                0.957 &                               0.957 &                              0.955 &                               0.963 &                              0.962 &                             0.964 &                                0.965 &                               0.965 &                              0.965 &                               0.974 &                              0.974 &                             0.974 \\
wdbc         &                                0.881 &                               0.881 &                              0.893 &                               0.917 &                              0.917 &                             0.917 &                                0.912 &                               0.912 &                              0.910 &                               0.941 &                              0.941 &                             0.933 \\
balance      &                                0.704 &                               0.703 &                              0.705 &                               0.701 &                              0.713 &                             0.717 &                                0.635 &                               0.635 &                              0.630 &                               0.801 &                              0.800 &                             0.801 \\
glass        &                                0.597 &                               0.597 &                              0.523 &                               0.633 &                              0.636 &                             0.589 &                                0.593 &                               0.601 &                              0.577 &                               0.633 &                              0.633 &                             0.628 \\
iris         &                                0.967 &                               0.967 &                              0.960 &                               0.960 &                              0.947 &                             0.947 &                                0.967 &                               0.967 &                              0.967 &                               0.967 &                              0.967 &                             0.960 \\
cleveland    &                                0.330 &                               0.331 &                              0.335 &                               0.309 &                              0.299 &                             0.351 &                                0.328 &                               0.315 &                              0.301 &                               0.318 &                              0.314 &                             0.359 \\
bupa         &                                0.608 &                               0.620 &                              0.612 &                               0.619 &                              0.605 &                             0.611 &                                0.625 &                               0.625 &                              0.625 &                               0.641 &                              0.641 &                             0.644 \\
haberman     &                                0.601 &                               0.598 &                              0.646 &                               0.607 &                              0.604 &                             0.642 &                                0.634 &                               0.634 &                              0.635 &                               0.628 &                              0.628 &                             0.627 \\
heart        &                                0.776 &                               0.776 &                              0.792 &                               0.820 &                              0.818 &                             0.813 &                                0.790 &                               0.790 &                              0.781 &                               0.820 &                              0.817 &                             0.814 \\
spectfheart  &                                0.604 &                               0.602 &                              0.602 &                               0.600 &                              0.600 &                             0.654 &                                0.749 &                               0.749 &                              0.742 &                               0.727 &                              0.730 &                             0.763 \\
\bottomrule
\end{tabular}
\end{adjustbox}
\caption{FGAC results for different $nn$ parameters}
\label{tab:results_approximation}
\end{table}

\begin{table}[]
    \centering
\begin{adjustbox}{width=\columnwidth}
    \begin{tabular}{l|rrrr|rrrr|rrrr}
\toprule
name &  
\makecell{fgac\_mae \\ supremum\_add} &  
\makecell{fgac\_mae \\ Euclidean\_add} &  
\makecell{fgac\_mse \\ supremum\_add} &  
\makecell{fgac\_mse \\ Euclidean\_add} &  
\makecell{fgac\_mae \\ supremum\_exp} &  
\makecell{fgac\_mae \\ Euclidean\_exp} &  
\makecell{fgac\_mse \\ supremum\_exp} &  
\makecell{fgac\_mse \\ Euclidean\_exp} &  
\makecell{fgac\_mae \\ supremum\_invadd} &  
\makecell{fgac\_mae \\ Euclidean\_invadd} &  
\makecell{fgac\_mse \\ supremum\_invadd} &  
\makecell{fgac\_mse \\ Euclidean\_invadd} \\
\midrule
australian   &                              0.840 &                             0.851 &                              0.842 &                             0.865 &                              0.846 &                             0.843 &                              0.846 &                             0.850 &                                 0.850 &                                0.849 &                                 0.853 &                                0.848 \\
breast       &                              0.537 &                             0.647 &                              0.537 &                             0.654 &                              0.537 &                             0.660 &                              0.537 &                             0.667 &                                 0.537 &                                0.654 &                                 0.537 &                                0.666 \\
crx          &                              0.726 &                             0.837 &                              0.727 &                             0.862 &                              0.740 &                             0.829 &                              0.738 &                             0.835 &                                 0.743 &                                0.835 &                                 0.769 &                                0.838 \\
flare        &                              0.529 &                             0.532 &                              0.528 &                             0.570 &                              0.562 &                             0.621 &                              0.528 &                             0.612 &                                 0.536 &                                0.624 &                                 0.528 &                                0.620 \\
german       &                              0.505 &                             0.680 &                              0.505 &                             0.695 &                              0.513 &                             0.669 &                              0.513 &                             0.657 &                                 0.513 &                                0.669 &                                 0.514 &                                0.682 \\
saheart      &                              0.658 &                             0.673 &                              0.677 &                             0.669 &                              0.672 &                             0.666 &                              0.689 &                             0.689 &                                 0.664 &                                0.665 &                                 0.675 &                                0.675 \\
ionosphere   &                              0.937 &                             0.939 &                              0.937 &                             0.937 &                              0.937 &                             0.937 &                              0.937 &                             0.939 &                                 0.937 &                                0.937 &                                 0.937 &                                0.937 \\
mammographic &                              0.802 &                             0.808 &                              0.803 &                             0.813 &                              0.802 &                             0.809 &                              0.802 &                             0.810 &                                 0.801 &                                0.808 &                                 0.808 &                                0.810 \\
pima         &                              0.700 &                             0.713 &                              0.696 &                             0.713 &                              0.704 &                             0.707 &                              0.713 &                             0.719 &                                 0.700 &                                0.709 &                                 0.705 &                                0.718 \\
wisconsin    &                              0.964 &                             0.969 &                              0.954 &                             0.956 &                              0.954 &                             0.966 &                              0.958 &                             0.973 &                                 0.961 &                                0.972 &                                 0.960 &                                0.962 \\
vowel        &                              0.945 &                             0.956 &                              0.957 &                             0.964 &                              0.949 &                             0.959 &                              0.961 &                             0.969 &                                 0.952 &                                0.956 &                                 0.960 &                                0.967 \\
wdbc         &                              0.868 &                             0.904 &                              0.900 &                             0.912 &                              0.904 &                             0.919 &                              0.918 &                             0.937 &                                 0.894 &                                0.918 &                                 0.905 &                                0.929 \\
balance      &                              0.769 &                             0.712 &                              0.784 &                             0.727 &                              0.716 &                             0.706 &                              0.722 &                             0.783 &                                 0.727 &                                0.705 &                                 0.745 &                                0.746 \\
glass        &                              0.523 &                             0.554 &                              0.574 &                             0.602 &                              0.585 &                             0.633 &                              0.640 &                             0.708 &                                 0.559 &                                0.588 &                                 0.627 &                                0.683 \\
iris         &                              0.967 &                             0.953 &                              0.973 &                             0.967 &                              0.967 &                             0.967 &                              0.973 &                             0.973 &                                 0.967 &                                0.953 &                                 0.973 &                                0.960 \\
cleveland    &                              0.346 &                             0.349 &                              0.327 &                             0.401 &                              0.319 &                             0.357 &                              0.310 &                             0.375 &                                 0.326 &                                0.353 &                                 0.342 &                                0.357 \\
bupa         &                              0.558 &                             0.568 &                              0.568 &                             0.607 &                              0.563 &                             0.579 &                              0.594 &                             0.606 &                                 0.563 &                                0.579 &                                 0.595 &                                0.601 \\
haberman     &                              0.595 &                             0.613 &                              0.641 &                             0.633 &                              0.630 &                             0.618 &                              0.636 &                             0.610 &                                 0.602 &                                0.607 &                                 0.630 &                                0.613 \\
heart        &                              0.785 &                             0.813 &                              0.781 &                             0.828 &                              0.792 &                             0.816 &                              0.795 &                             0.822 &                                 0.798 &                                0.815 &                                 0.798 &                                0.833 \\
spectfheart  &                              0.557 &                             0.580 &                              0.630 &                             0.639 &                              0.585 &                             0.637 &                              0.732 &                             0.732 &                                 0.578 &                                0.604 &                                 0.701 &                                0.703 \\
\bottomrule
\end{tabular}
\end{adjustbox}
\caption{FGAC results for different OWA weights}
\label{tab:results_owa}
\end{table}

\end{landscape}

\subsection{Comparison of FGAC with other ML methods}

We first discuss how the hyperparameters are tuned. We already stated previously that every model depends on one parameter and we tune that parameter using 5-fold cross-validation. They are selected from a finite set of values based on their performance. In the following table, we list the models and the corresponding sets of possible values of their hyperparameters.
\begin{table}[H]
\centering
\begin{tabular}{|l|l|}
\hline
models & possible hyperparameter values                             \\ \hline
FGAC   & \{0.5, 0.7, 0.8, 1, 1.5, 2, 3, 4, 5, 7, 10 \}              \\ \hline
kFRNN  & \{ all, 1, 3, 5, 10, 15, 20, 25, 30, 40, 50\}               \\ \hline
kNN    & \{1, 3, 5, 7, 10, 15, 20, 25, 30, 40, 50\}                 \\ \hline
LVQ    & \{1,2,3,4,5,6,7,8, 9,10, 11\}                              \\ \hline
CART     & \{2,3,4,5,6,7,8, 9,10, 11,12\}                             \\ \hline
\end{tabular}
\end{table}
The possible values are constructed based on the preliminary analysis. Every model is provided with 11 possible hyperparameters. Value "all" in the kFRNN hyperparameters set indicates that the OWA weights were applied on all instances. Also, after preliminary analysis, we concluded that the best performing version of kFRNN is the one with OWA additive weights and that uses Euclidean similarity and hence, it is used in the comparison process as the representative of kFRNN. 

In Table \ref{tab:final_results} we show the performances of the models. In every row, with the black bold font, we label the best performing model, with the red bold font the second one, while with the blue bold font the third best. After running the Friedman test on the results, we obtain a $p$-value equal to 0.008 which implies that the models are significantly different.

In the next table, we show the average rankings of these models.
\begin{table}[H]
\centering
\begin{tabular}{|l|l|l|l|l|l|}
\hline
models       & FGAC & kFRNN   & kNN  & LVQ  & CART   \\ \hline
average rank & 3.39 & 2.06   & 2.67 & 2.94 & 3.94 \\ \hline
\end{tabular}
\end{table}

First, we observe that FGAC has the second worst performance based on the average rank; it is only better than CART. However, if we apply the Wilcoxon test to make pairwise comparisons of FGAC with the remaining models, we obtain the $p$-values in Table \ref{tab:pairwise_comparison}.

\begin{table}[H]
    \centering
    \begin{tabular}{lrrrrr}
\toprule
name &  FGAC &  kFRNN  &  kNN &  LVQ &  CART \\
\midrule
australian   &  0.862 &   \boldblue{0.873} &          \textbf{0.885} &      \boldred{ 0.875} &     0.861 \\
breast       &                             \textbf{0.673} &                    \boldblue{ 0.663} &      \boldblue{ 0.663} &      \boldred{0.670} &     0.630 \\
crx          &  0.827 &  \textbf{ 0.879} &     \boldred{0.877} &  \boldblue{0.872} &     0.870 \\
german       &    0.656 &    \textbf{ 0.704} &     \boldblue{  0.685 }&      \textbf{0.704} &     0.674 \\
saheart      &                            \boldblue{ 0.677} &                     \boldred{ 0.679} &       0.676 &      \textbf{0.694} &     0.660 \\
ionosphere   &                            \textbf{ 0.938} &                      0.846 &   \boldblue{ 0.849} &      0.846 &    \boldred{ 0.892} \\
mammographic &                             0.807 &                      0.815 &     \boldred{ 0.823} &     \boldblue{ 0.812} &    \textbf{ 0.832} \\
pima         &                            \boldblue{ 0.736} &                      \textbf{0.737} &     \boldblue{ 0.736} &      \textbf{0.737} &     0.721 \\
wisconsin    &   \boldblue{0.972} &  \boldred{0.979} &      \textbf{ 0.981} &      0.970 &     0.948 \\
vowel        &  \boldblue{0.976} &                      \textbf{0.988} &    \textbf{0.988} &      0.728 &     0.799 \\
wdbc         &                             0.935 &                      \textbf{0.969} &    \boldred{ 0.966} &     \boldblue{ 0.945} &     0.929 \\
balance      &   \textbf{ 0.810} &  \boldred{0.768} &    \boldblue{ 0.750} &      0.654 &     0.681 \\
glass        &        0.639 &                    \boldred{0.684} &     \textbf{0.686} &      \boldred{0.684} &     0.668 \\
cleveland    &      \boldblue{  0.359} &                     \textbf{ 0.424} &      \boldred{ 0.414} &      0.347 &     0.289 \\
bupa         &         0.640 &      \textbf{ 0.655} &       0.636 &    \boldblue{  0.647} &     \boldred{0.650} \\
haberman     &        0.628 &       0.640 &    \boldred{ 0.645} &      \boldblue{0.644} &    \textbf{ 0.663} \\
heart        &       0.821 &     \boldred{  0.837} &         \boldblue{0.835} &    \textbf{  0.842} &     0.804 \\
spectfheart  &    \boldred{   0.763} &                      \boldblue{0.732} &    0.731 &     \textbf{ 0.780} &     0.675 \\
\bottomrule
\end{tabular}
    \caption{Comparison of the FGAC with the other ML models based on the balanced accuracy}
    \label{tab:final_results}
\end{table}

\begin{table}[H]
\centering
\begin{tabular}{|l|l|l|l|l|}
\hline
            & kFRNN & kNN   & LVQ   & CART  \\ \hline
$p$-values: & 0.098 & 0.246 & 0.347 & 0.167 \\ \hline
\end{tabular}
\caption{Pairwise comparison of the FGAC with other models}
\label{tab:pairwise_comparison}
\end{table}
We observe that even though some methods have a higher average rank than FGAC, we cannot claim that they are indeed significantly better (all values are larger than 0.05). In the same manner, we cannot claim that FGAC is significantly better than CART. 

In the next section, we discuss the greatest advantage of FGAC - its transparency.

\section{Transparency}
\label{sec:transparency}
In this section we discuss the transparency of the proposed FGAC and we compare it with the transparency of the other methods. We distinguish two types of transparent models, those that are globally transparent and those that are locally. Global transparency is achieved when the model as a whole can be explained and understood. Local transparency occurs when the individual predictions separately can be comprehended. We claim that FGAC can be considered as a part of both families.

In this section, we first discuss the method from the perspective that fuzzy logic can be translated into linguistic expressions. The second part of the section is related to identifying the arguments "in favour" and "against" the estimated membership degree of an individual instance. This is a form of local transparency. At the end, we compare the local transparency of FGAC with the ML methods from Section \ref{subsec:models_to_compare_with}.

\subsection{Fuzzy logic and linguistics}
The goal of this section is to interpret the expression (\ref{eq:bounds}) and its multi-class version, i.e., we will explain these inequalities by utilizing the ability to express the fuzzy connectives using plain words. We interpret a $T$-equivalence relation as "similarity", $t$-norms as the "and" connective and implicators as IF-THEN rules. 

First, we interpret the well-definedness of the bounds expressed through Proposition \ref{prop:well_defined} as well as the proof of the proposition.

An equivalent form of the well-definedness of the bounds is given in (\ref{eq:bounds_well_defined2}). For some $u,v \in U$, the interpretation of that expression is:
\begin{equation}
\label{eq:bounds_well_defined_exp}
  \text{IF $u \sim u^{\dagger}$ and $u \, \widetilde{\in}\, A$ THEN IF $v \sim u^{\dagger}$ THEN $v \, \widetilde{\in}\, A$}, 
\end{equation}
where $\sim$ means "is similar to" and $\widetilde{\in}$ stands for fuzzy membership, i.e., we read it as "belongs to". Therefore, we read the previous expression as "If $u$ is similar to $u^{\dagger}$ and $u$ belongs to $A$ then, if $v$ is similar to $u^{\dagger}$ then $v$ is in $A$".

Following the proof of the proposition, the previous expression is equivalent to (residuation property):
\begin{equation*}
     \text{
     IF $ u \sim u^{\dagger}$ and $v \sim u^{\dagger}$ and $u \, \widetilde{\in}\, A$ THEN $v \, \widetilde{\in}\, A$,
    }
\end{equation*}
which is true from the $T$-transitivity of $\sim$ and the granularity property. 
Since expression (\ref{eq:bounds_well_defined_exp}) holds for all $u$ and $v$, it can be translated to:
\begin{equation*}
  \text{IF  $\exists u \in U$ s.t. $u \sim u^{\dagger}$ and $u \, \widetilde{\in}\, A$ THEN $\forall  v \in U$ IF $v \sim u^{\dagger}$ THEN $v \, \widetilde{\in}\, A$}. 
\end{equation*}
Here, the symbols $\exists$ and $\forall$ have their usual meanings: "there exists" and "for all" respectively, while "s.t." is the abbreviation for "such that".
Putting back the membership degree of $u^{\dagger}$, the two inequalities of (\ref{eq:bounds}) can be interpreted as follows. For the left inequality we have:
\begin{equation}
\label{eq:left_inequality_words}
  \text{IF $\exists u \in U$ s.t. $u \sim u^{\dagger}$ and $u \, \widetilde{\in}\, A$, THEN $u^{\dagger} \, \widetilde{\in}\, A$},
\end{equation}
while for the right inequality, we have that:
\begin{equation}
\label{eq:right_inequality_words}
  \text{IF $u^{\dagger} \, \widetilde{\in}\, A$, THEN $\forall  v \in U$, IF $v \sim u^{\dagger}$ THEN $v \, \widetilde{\in}\, A$}.
\end{equation}
We apply the previous expressions on our example with the movie streaming service. From (\ref{eq:left_inequality_words}) we have that: if there exists a movie $u$ that is similar to movie $u^{\dagger}$ and the user likes movie $u$, then the user will also like movie $u^{\dagger}$. From (\ref{eq:right_inequality_words}) we have that: if the user likes movie $u^{\dagger}$ then they should also like all movies that are similar to $u$.

\subsection{Arguments for the classification}
The next step is to identify and to interpret the training instances based on which the decision for a new instance was made. These instances are argmax from the left equation and argmin from the right equation in (\ref{eq:bounds}). The argmax is the instance that supports the decision $u^{\dagger} \in A$ since it is at the same time the most similar to $u^{\dagger}$ and has the highest estimated membership in $A$. All other instances are either less similar to $u^{\dagger}$ or less present in $A$. Hence, the argmax is the argument in favour of decision  $u^{\dagger} \in A$. The argmin is the instance that objects the decision $u^{\dagger} \in A$ since it supports the decision $u^{\dagger} \in coA$. This is visible by applying negator $N$ on the right inequality of (\ref{eq:bounds}) and obtaining
$N(\hat{A}(u^{\dagger})) \geq \max_{u \in U} T(\widetilde{R}(u, u^{\dagger}), N(\hat{A}(u)))$. After obtaining the previous expression, we can use the reasoning from above to justify that the argmin indeed supports $u^{\dagger} \in coA$, i.e., objects $u^{\dagger} \in A$. In other words, the argmin is the argument against the decision $u^{\dagger} \in A$.

The conclusion of the previous paragraph is that we are able to find arguments in favour of the decision, as well as arguments against the decision that we are making. If we need more than one argument for the decision, we can consider a few top instances (not only minimum and maximum) that support and that object the decision. In our example of movie recommendations, for every movie for which we predict the degree of allure to the user, we can identify the movies that support this degree and the movies that object the degree from the movies that user already watched and rated. Moreover, for arguments that are in favour of a decision, value $T(\widetilde{R}( u^{\dagger}, u), \hat{A}(u))$ can be seen as the strength of the argument. The greater the strength, the more confident we are in our decision. On the other hand, for arguments that go against the decision, value $T(\widetilde{R}(u, u^{\dagger}), N(\hat{A}(u)))$ can be seen as the strength of the argument. If the value is greater, then value $I(\widetilde{R}(u, u^{\dagger}), \hat{A}(u))$ is smaller which further implies that the confidence in our decision is also smaller. 

Since we are able to precisely identify the arguments based on which the decision was made and since those arguments can be well comprehended by a human, we may say that FGAC is fully locally transparent. 

\subsection{Transparency comparison with other models}

We compare the transparency of the proposed FGAC with the other known ML models from Section \ref{subsec:models_to_compare_with}. These models can be divided into three groups: instance-based (kFRNN and kNN), prototype-based (LVQ) and rule-based (CART). All these types of models possess some form of local transparency and this is the reason they are selected for this comparison experiment.

There are other transparent methods like linear models (e.g. logistic regression) but based on the way they are interpreted in practice, they can be classified as globally transparent models and therefore not really comparable with FGAC.


In the case of CART, for every performed classification, we are able to identify the corresponding decision rule from the tree structure of the classifier based on which the classification is performed. On the global level, the set of all decision rules can be seen as a form of global transparency. However, in practice, the number of rules can be very large which aggravates the understanding of the model as a whole. If the number of rules is kept relatively small (e.g. less than 10), we may say that we also achieve global transparency. On the other side, decision rules depend on the attributes used in the modeling and any feature engineering process may affect the transparency of CART. Since FGAC is instance-based, it is not dependent on the attribute space used for modeling and therefore, can be seen as superior in that context. However, the interpretation of rules has its advantages in a way that we are able to exactly identify the way one attribute affects the final decision.

For the LVQ method, 
we observe that during the training phase, few points in the attribute space are learned as prototypes for every decision class. Later on, the decision is made based on the closest prototype. Prototype-based and instance-based (like FGAC) methods share similarities in a way that both methods make predictions based on the closest points from the attribute space. The difference is that in prototype-based methods, these points are not from the set of training instances, but they can be any point from the space. This is a huge disadvantage if a certain amount of feature engineering is applied and the original attribute space is changed: the learned prototypes loose their meaning and the method becomes non-transparent. On the other hand, the transparency of instance-based methods does not depend on feature engineering. Therefore, the transparency of FGAC is more advantageous compared that of LVQ.

Now we move to the remaining methods, kNN and kFRNN that are both instance-based, i.e., of the same type as FGAC. Their possible transparency lies in identifying instances based on which a prediction was made. Their transparency heavily depends on the number of instances used for prediction making, i.e., hyperparameter $k$. If $k$ is high, it is really hard to identify how the prediction is made. We observed that during training of the kNN and kFRNN, the majority of performances from Table \ref{tab:final_results} are achieved for higher values of $k$ ($k > 5$) which means that in the majority of cases, the prediction process in both kNN and kFRNN is barely transparent. Also, kNN and kFRNN are not significantly better than FGAC according to Table \ref{tab:pairwise_comparison}.

Now, we want to compare the FGAC with the more transparent variants of the kNN and the kFRNN. We consider a similar transparency level as for FGAC i.e., $k=1$ and a less transparent case when $k \leq 5$. The comparison results are shown in Table \ref{tab:transparent_comparison}. Bold values indicate the best performing model.
\begin{table}[H]
    \centering
    \begin{adjustbox}{width=\columnwidth,center}
    \begin{tabular}{lrrrrr}
\toprule
{} &  FGAC &  kNN ($k\leq 5$) &  kFRNN ($k\leq 5$) &  kNN ($k = 1$) &  kFRNN ($k=1$) \\
\midrule
australian   &                            \textbf{0.862} &      0.837 &                      0.835 &      0.809 &                      0.801 \\
breast       &                           \textbf{  0.673} &      0.658 &                      0.536 &      0.638 &                      0.524 \\
crx          &                             0.827 &     \textbf{ 0.859} &                      0.738 &      0.819 &                      0.723 \\
german       &                             \textbf{0.656} &      0.640 &                      0.509 &      0.637 &                      0.509 \\
saheart      &                            \textbf{ 0.677} &      0.629 &                      0.617 &      0.583 &                      0.593 \\
ionosphere   &                            \textbf{ 0.938} &      0.861 &                      0.854 &      0.823 &                      0.835 \\
mammographic &                            \textbf{ 0.807} &      0.805 &                      0.788 &      0.745 &                      0.738 \\
pima         &                            \textbf{ 0.736} &      0.697 &                      0.700 &      0.672 &                      0.649 \\
wisconsin    &                             0.972 &    \textbf{  0.981} &                      0.970 &      0.954 &                      0.928 \\
vowel        &                             0.976 &     \textbf{ 0.988} &                      0.974 &      0.988 &                      0.974 \\
wdbc         &                             0.935 &    \textbf{  0.967} &                      0.945 &      0.950 &                      0.936 \\
balance      &                           \textbf{ 0.810} &      0.562 &                      0.515 &      0.538 &                      0.469 \\
glass        &                             0.639 &    \textbf{  0.682} &                      0.631 &      0.674 &                      0.555 \\
cleveland    &                           \textbf{  0.359} &      0.314 &                      0.321 &      0.305 &                      0.303 \\
bupa         &                            \textbf{ 0.640} &      0.628 &                      0.633 &      0.619 &                      0.622 \\
haberman     &                          \textbf{ 0.628} &      0.587 &                      0.551 &      0.556 &                      0.522 \\
heart        &                           \textbf{ 0.821} &      0.798 &                      0.767 &      0.749 &                      0.719 \\
spectfheart  &                           \textbf{  0.763} &      0.695 &                      0.660 &      0.600 &                      0.651 \\
\bottomrule
\end{tabular}
\end{adjustbox}
    \caption{Comparison of FGAC with the transparent versions of the kNN and kFRNN}
    \label{tab:transparent_comparison}
\end{table}
After applying the Friedman test on the results in Table \ref{tab:transparent_comparison}, we get a $p$-value of order $10^{-9}$ which means that the performances are indeed different. From the table, we observe that FGAC is the best model in the most occurrences. Using Holm post-hoc analysis, we test if FGAC is indeed the best model and we get that the $p$-value is equal to 0.034. This means that FGAC is indeed the best performing model among the selected transparent instance-based classifiers.

\section{Conclusion and future work}
\label{sec:conclusion}
In this paper we introduced a Fuzzy Granular Approximation Classifier (FGAC) based on granular approximations introduced in \cite{palangetic2021granular} and \cite{palangetic2022multi}. We also introduced a version that uses OWA operators. Furthermore, we discussed ways to speed up the training of the classifier. The empirical comparisons led to the following main conclusions:
\begin{itemize}
    \item The best performing version of FGAC is the one that uses MSE as the loss function and the Euclidean similarity.
    \item Adding OWA operators does not change the performance of FGAC.
    \item In comparison with other models, FGAC outperformed only CART based on the average rank. However, after pairwise significance testing with other models, no other model outperformed FGAC significantly. 
\end{itemize}
Later, we showed that FGAC can be described using plain words due to the linguistic nature of fuzzy logic. The method is also fully locally transparent where for every prediction we are able to identify the arguments for that prediction that are both in favour and against.
Finally, we showed that FGAC is more advantageous compared to other models regarding its local transparency.

We consider the following possibilities for the future work:
\begin{itemize}
    \item In this paper, we used a $T$-equivalence relation that is suitable for ordinal classification problems. On the other side, using non-symmetric $T$-preorder relation is more suitable for the monotone classification problems. Since the binary version of the FGAC is developed also for the non-symmetric relations, we would like to explore its performance in monotone classification problems.
    \item In this paper, we used fuzzy connectives based on the Łukasiewicz $t$-norm. In the future, we would like to explore if using different fuzzy connectives, isomorphic to the Łukasiewicz ones, or in general different fuzzy connectives, can lead us to the better results. 
    \item We also want to explore if the FGAC can be extended to the regression problems.
\end{itemize}



\bibliographystyle{elsarticle-num} 
\bibliography{bibliography}


\end{document}